\documentclass[twoside]{article}

\usepackage[accepted]{aistats2020}
\usepackage{subcaption}
\usepackage[utf8]{inputenc} 
\usepackage[T1]{fontenc}    
\usepackage[hidelinks]{hyperref}       
\usepackage{url}            
\usepackage{booktabs}       
\usepackage{amsfonts}       
\usepackage{nicefrac}       
\usepackage{microtype}      
\usepackage{algorithm}
\usepackage{algpseudocode}
\usepackage{algorithmicx}
\usepackage{amsmath}
\usepackage{amsthm}
\usepackage{xspace}
\usepackage{amssymb}
\usepackage{color}
\usepackage{tikz}
\usepackage{natbib}
\usetikzlibrary{arrows}	
\usepackage{pgfplots}
\usepgfplotslibrary{groupplots} 
\usetikzlibrary{pgfplots.groupplots}
\usetikzlibrary{matrix,arrows,decorations.pathmorphing}

\newcommand{\ssp}{\vspace{-0.5mm}}
\newcommand{\msp}{\vspace{-1.2mm}}
\newcommand{\lsp}{\vspace{-2.0mm}}
\usepackage{bm}

\newcommand{\ben}{\begin{enumerate*}}
\newcommand{\een}{\end{enumerate*}}
\newcommand{\beq}{\begin{eqnarray}}
\newcommand{\eeq}{\end{eqnarray}}
\newcommand{\bit}{\begin{itemize*}}
\newcommand{\eit}{\end{itemize*}}

\newcommand{\hide}[1]{}

\newtheorem{theorem}{Theorem} 

\newcommand{\indep}[1]{\perp}

\DeclareMathOperator*{\argmax}{arg\,max}

\newcommand{\appropto}{\mathrel{\vcenter{
  \offinterlineskip\halign{\hfil$##$\cr
    \propto\cr\noalign{\kern2pt}\sim\cr\noalign{\kern-2pt}}}}}

\renewcommand{\eqref}[1]{Equation~(\ref{#1})}

\newtheorem{corollary}[theorem]{Corollary}
\newtheorem{prop}[theorem]{Proposition}
\newtheorem{lemma}[theorem]{Lemma}
\newcommand{\E}{\mathbb{E}}
\newcommand{\narm}{K}
\newcommand{\tol}{\delta}
\newcommand{\sammean}[1]{\mu_{#1}}
\newcommand{\estmean}[1]{\hat{\mu}_{#1}}
\newcommand{\pbbeat}[1]{p_{#1}}
\newcommand{\estprob}[1]{\hat{p}_{#1}}
\newcommand{\dlabel}{y}

\newcommand{\dlabelrv}{Y}
\newcommand{\comprv}{Z}
\newcommand{\ncomp}[1]{n_{#1}}
\newcommand{\nwin}[1]{w_{#1}}
\newcommand{\sgpara}{R}
\newcommand{\samiter}{i}
\newcommand{\samiterp}{j}
\newcommand{\dlcpl}{H_l}
\newcommand{\thres}{\tau}
\algrenewcommand\algorithmicrequire{\textbf{Input:}}
\algrenewcommand\algorithmicensure{\textbf{Output:}}
\newcommand{\nlabel}[1]{m_{#1}}
\newcommand{\bandsum}[1]{s_{#1}}
\newcommand{\problemname}{TBP-DC\xspace}
\newcommand{\algoname}{\textsf{RS}\xspace}

\DeclareMathOperator*{\argmin}{arg\,min}
\newcommand{\posset}{S_{\thres}}
\newcommand{\negset}{S^c_{\thres}}
\newcommand{\estposset}{\hat{S}_{\thres}}
\newcommand{\estnegset}{\hat{S}^c_{\thres}}
\newcommand{\labelgap}[1]{\Delta^l_{#1}}
\newcommand{\compgap}[1]{\Delta^c_{#1}}

\definecolor{DSgray}{cmyk}{0,1,0,0}

\newcommand{\itercount}{t}
\newcommand{\bsvar}{k} 

\newcommand{\comperr}{\gamma}
\newcommand{\workset}{S}

\newcommand{\bs}{i}
\newcommand{\upthres}{\samiter_u}
\newcommand{\lowthres}{\samiter_l}
\newcommand{\bsalgoname}{\textsf{Binary-Search}\xspace}
\newcommand{\figurelabelname}{\textsf{Figure-Out-Label}\xspace}

\newcommand{\algo}{\mathcal{A}}
\newcommand{\meanvec}{\bm{\mu}}
\newcommand{\dueltime}[1]{D_{#1}}
\newcommand{\pulltime}[1]{L_{#1}}
\newcommand{\event}{\mathcal{E}}
\newcommand{\samiterint}{k}
\newcommand{\numduel}{n_{\text{duel}}}
\newcommand{\numpull}{n_{\text{pull}}}

\begin{document}

%

%
\runningauthor{Yichong Xu, Xi Chen, Aarti Singh, Artur Dubrawski}
\twocolumn[

\aistatstitle{Thresholding Bandit Problem with Both Duels and Pulls}

\aistatsauthor{Yichong Xu$^1$, Xi Chen$^2$ Aarti Singh$^1$, Artur Dubrawski$^1$}

\aistatsaddress{$^1$Machine Learning Department, Carnegie Mellon University
$^2$Stern School of Business, New York University} 
]

\begin{abstract}
  The Thresholding Bandit Problem (TBP) aims to find the set of arms with mean rewards greater than a given threshold. We consider a new setting of TBP, where in addition to pulling arms, one can also \emph{duel} two arms and get the arm with a greater mean. In our motivating application from crowdsourcing,  dueling two arms can be more cost-effective and time-efficient than direct pulls. We refer to this problem as TBP with Dueling Choices (TBP-DC). This paper provides an algorithm called Rank-Search (\algoname) for solving TBP-DC by alternating between ranking and binary search. We prove theoretical guarantees for \algoname, and also give lower bounds to show the optimality of it. Experiments show that \algoname outperforms previous baseline algorithms that only use pulls or duels.
\end{abstract}

\lsp
\section{Introduction}
\ssp
The Thresholding Bandit Problem (TBP, \citep{locatelli2016optimal}) is an important pure-exploration multi-armed bandit (MAB) problem. Specifically, given a set of $\narm$ arms with different mean rewards, the TBP aims to find arms whose mean rewards are above a pre-set threshold of $\thres$. The TBP has a wide range of applications, such as anomaly detection, candidate filtering, and crowdsourced classification. For example,  a popular crowdsourced classification model \citep{Chen:16KL,chen2015statistical} assumes that there are $\narm$ items with the latent true labels $\theta_i \in \{0, 1\}$ for each item. The labeling difficulty of the $i$-th item is characterized by its soft label $\mu_i \in [0,1]$, which is defined as the probability that a random crowd worker will label the $i$-th item as positive.  It is clear that the item is easy to label when $\mu_i$ is close to 0 or 1, and difficult when $\mu_i$ is close to 0.5. In MAB, $\mu_i$ is the mean reward of arm $i$, and pulling this arm leads to a Bernoulli observation with mean $\mu_i$. Moreover, it is natural to assume that the soft label $\mu_i$ is consistent with the true label, i.e., $\mu_i \geq 0.5$  if and only if $\theta_i=1$. Therefore,  identifying items belonging to class $1$ is equivalent to detecting those arms with $\mu_i> \thres$ with $\thres=0.5$. 

Existing literature on TBP considers the setting that only solicits information from pulling arms directly. However, in many applications of TBP, comparisons/duels can be obtained at a much lower cost than direct pulls. In crowdsourcing, a worker often compares two items more quickly and accurately than labeling them separately. It will be cheaper and time efficient to ask a worker which image is more relevant to a query as compared to asking for an absolute relevance score of an image (see, e.g., \cite{shah2016estimation}). Another example is in material synthesis, a pull will need an expensive synthesis of the material, whereas duels can be carried out easily by querying experts. In such settings, directly pulling an arm is expensive and could incur a large sample complexity since each arm needs to be pulled a number of times. This paper considers two sources of information: in addition to direct pulls of arms as in the classical TBP, one can also \emph{duel} two arms to find out the arm with a greater mean at a lower cost. We refer to this problem as the TBP with Dueling Choices (\problemname), since dueling and pulling are both available in each round.  



It is important to note that some direct pulls are still necessary for solving a TBP even if one can duel two arms. Without direct assessments of arms, we can at best rank all the arms with duels. However,  we then cannot know the target threshold $\thres$ and therefore cannot identify a boundary on the ranking.
On the other hand, using an appropriate dueling strategy, the number of required direct pulls can be much lower than that in the classical TBP setting, where only direct pulling is available. We further note that \problemname is also different from the top-K arm identification problem considered in whether MAB (see, e.g., \cite{Bubeck:13}, \cite{Chen:14PAC}, \cite{Chen:17adaptive}) or dueling bandits (see, e.g., \cite{mohajer2017active}), because the number of arms with means greater than the threshold $\thres$ is unknown to us. 

A straightforward way to solve the TBP-DC problem is to utilize an existing ranking algorithm such as ActiveRank \citep{heckel2016active} to rank all the arms, and then use a binary search to find the boundary. However, this method is impractical because it can be very hard to differentiate arms with similar means (e.g., equally good images, similar quality materials). These arms might be far from the threshold and it is actually unnecessary to differentiate them. We instead take an iterative approach; We develop the Rank-Search (\algoname) algorithm for \problemname, which alternates between refining the rank over all items using duels and a binary search process using pulls to figure out the threshold among ranked items. We interleave the ranking and searching step so that we do not waste time differentiating equally good arms.

\textbf{Our contributions.} 
First, in Section~\ref{sec:algo}, we analyze the number of duels and pulls required for \algoname under the fixed confidence setting, i.e., to recognize the set of arms with reward larger than $\thres$ with probability at least $1-\tol$.  To better illustrate our main idea, we further provide concrete examples, which show that the proposed \algoname only requires $O(\log^2 K)$ direct labels, while the classical TBP requires at least $\Omega(K)$ labels (see Section~\ref{sec:example_upper}). Section~\ref{sec:lower} shows  complementary lower bounds that \algoname is near-optimal in both duel and pull complexity. Finally, we provide practical experiments to demonstrate the performance of \algoname.


\textbf{Related Works.} TBP is a special case of the pure-exploration combinatorial MAB problem. As with other pure-exploration MAB problems\citep{Bubeck:13} , algorithms for combinatorial bandits fall into either \emph{fixed-budget} or \emph{fixed-confidence} categories. In the former setting, the algorithm is given a time horizon of $T$ and tries to minimize the probability of failure. In the latter setting, the algorithm is given a target failure probability and tries to minimize the number of queries. For TBP, the CLUCB algorithm~\citep{chen2014combinatorial}  can solve TBP under the pull-only and fixed confidence setting, with optimal sample complexity. \citep{chen2014combinatorial} also develops the CSAR algorithm for the fixed-budget setting which can also be used for TBP. The result was improved by recent followup work~\citep{locatelli2016optimal,mukherjee2017thresholding} under the fixed budget setting. \cite{chen2015statistical} considered TBP in the context of budget allocation for crowdsourced classification in the Bayesian framework. 

Motivated by crowdsourcing and other applications, this paper proposes a new setup since we allow both pulling one arm and dueling two arms in each round, with the underlying assumption that dueling is more cost-effective than pulling. To the best of our knowledge, this setting has not been considered in the previous work.  Most close in spirit to our work is a series of recent papers \citep{kane2017active,xu2018nonparametric,xu2017noise}, which consider using pairwise comparisons for learning classifiers. The methods in those papers are however not directly applicable to \problemname because their final goal is to learn a classification boundary, instead of labeling each item without feature information. 

\newcommand{\armset}{\mathcal{A}}
\newcommand{\samdistr}[1]{\nu_{#1}}
\newcommand{\timeiter}{t}
\newcommand{\prefmat}{M}
\newcommand{\samgap}[1]{\Delta_{#1}}
\msp
\section{Problem Setup}
\ssp
Suppose there are  $\narm$ arms,  which are denoted by  $\armset=[\narm]=\{1,2,...,\narm\}$. Each arm $\samiter\in \armset$ is associated with a mean reward $\sammean{\samiter}$. Without loss of generality, we will assume that $\sammean{1}\leq \sammean{2}\leq \cdots \leq \sammean{\narm}$. Given a target threshold $\thres$, our goal is to identify the positive set $\posset=\{i:\sammean{i}\geq \thres\}$ and the negative set $\negset=\{i:\sammean{i}<\thres\}$.  

\textbf{Modes of interactions.} Each instance of \problemname is uniquely defined by the tuple $(\prefmat,\meanvec)$, where $\prefmat$ is the preference matrix (defined below) and $\meanvec=\{\sammean{\samiter} \}_{\samiter=1}^\narm$ is the mean reward vector. 
In each round of our algorithm, we can choose one of two possible interactions:\\
\ssp
\begin{itemize}
	\item \msp
	\textbf{Direct Queries (Pulls)}: We choose an arm $\samiter\in \armset$ and get a (independent) noisy reward $Y$ from arm $\samiter$. We assume that each arm $\samiter$ is associated a reward distribution $\samdistr{\samiter}$ with mean $\sammean{\samiter}$, and that $\samdistr{\samiter}$ is 
	sub-Gaussian with parameter $\sgpara$:
	\(\E_{Y\sim \samdistr{\samiter}}[\exp(t\dlabelrv - t\E[\dlabelrv])] \leq \exp(\sgpara^2t^2/2)\)
	for all $t\in \mathbb{R}$. 
	The definition of sub-Gaussian variables includes many common distributions, such as Gaussian distributions or any bounded distributions (e.g., Bernoulli distribution). We denote by $\labelgap{\samiter}=|\sammean{\samiter}-\thres|$ the gap between arm $\samiter$ and the threshold.\\
	\item \msp
	\textbf{Comparisons (Duels)}: We can also choose to duel two arms $\samiter, \samiterp\in \armset$ and obtain a random variable $\comprv$, with $\comprv=1$ indicating the arm $\samiter$ has a larger mean reward than $\samiterp$ and $\comprv=0$ otherwise. Let $\prefmat_{\samiter\samiterp}\in [0,1]$ characterize the probability that a random worker believes that  arm $\samiter$ is ``more positive'' than arm $\samiterp$. The outcome of duels is therefore characterized by the matrix $\prefmat$. The (Borda) score of each arm in dueling is defined as 
	\begin{equation}\label{def:p}
	\pbbeat{\samiter}:=\frac{1}{\narm-1}\sum_{\samiterp\in [\narm]\setminus \{\samiter\}} M_{\samiter\samiterp},
	\end{equation}
	i.e., the probability of arm $\samiter$ beating another randomly chosen arm $\samiterp$. 
	
	In contrast to previous work \citep{shah2016estimation,szorenyi2015online,yue2012k} that usually assumes parametric or structural assumptions on $\prefmat$, we allow an arbitrary preference matrix $\prefmat$; the only assumption is that the score of any positive arm is larger than any negative arm, i.e., $\pbbeat{\samiter}>\pbbeat{\samiterp}, \forall \samiter\in \posset, \samiterp\in \negset$.
	We note that this is a very weak condition since arbitrary relations within the positive and negative sets are allowed. This assumption also holds if $(1,2,...,\narm)$ is the Borda ranking of $\prefmat$; or the underlying comparison model follows the Strong Stochastic Transitivity (SST, \citep{fishburn1973binary,shah2016stochastically}). 
	We note that the problem is very difficult under this assumption: For example, even if $\mu_i$ (knowledge from pulls) are bounded away from $\tau$ by a constant, the $p_i$ (knowledge from duels) may be arbitrarily close, hence making the problem much harder.
\end{itemize}
\msp

Taking crowdsourced binary classification as an example, $\dlabelrv_i \in \{0,1\}$ would correspond to a binary label of the $i$-th item obtained from a worker, where $\sammean{\samiter}=\Pr_{\dlabelrv_i \sim \samdistr{\samiter}}[\dlabelrv=1]$. For this case we have $\thres=1/2$. Dueling outcome $\comprv_{ij}$ will correspond to asking a worker to compare item $\samiter$ with item $\samiterp$ and $\comprv_{ij}=1$ if the worker claims that item $\samiter$ is ``more positive'' than item $\samiterp$.

\textbf{The fixed-confidence setting.} Given a target error rate $\tol$, our goal is to recover the sets $\estposset$ and $\estnegset$, such that $\Pr[\posset=\estposset, \negset=\estnegset]\geq 1-\tol$, with as fewer pulls and duels as possible. Since in practice duels are often cheaper than pulls, we want to minimize the number of pulls while also avoiding too many duels.

\msp
\subsection{Problem Complexity}
\ssp

\newcommand{\score}{\tau}
\newcommand{\gaplo}{\Delta}
\newcommand{\gapup}{\delta}
\newcommand{\diffcomp}{\bar{\Delta}^c}
\newcommand{\compgood}{T}
\newcommand{\compcpl}{H_{c,2}}
\newcommand{\compcpllow}{H_{c,1}}

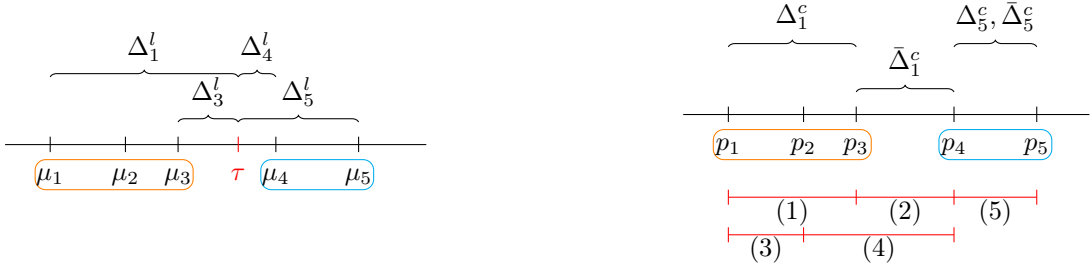
\begin{figure*}
	\centering
	\begin{subfigure}[b]{0.48 \textwidth}
		\centering
		\begin{tikzpicture}[>=latex,scale=1.0,baseline={(0.5,0.4)}]
		\def\recw{0.2}
		\draw [rounded corners,orange] (2.7+\recw, -\recw)--(2.7+\recw,+\recw)--(1-\recw, +\recw)--(1-\recw,-\recw)--cycle;
		\draw [rounded corners,cyan] (5.1+\recw, -\recw)--(5.1+\recw,+\recw)--(4.0-\recw, +\recw)--(4.0-\recw,-\recw)--cycle;    
		\draw (0.4,0.4) -- (6.0,0.4); 
		\foreach \x/\xtext in {1/\sammean{1}, 2/\sammean{2},
			2.7/\sammean{3}, 4.0/\sammean{4}, 5.1/\sammean{5}}{ 
			\draw (\x,0.4cm+3pt) -- (\x,0.3);
			\node at (\x,0.2) [anchor=north] {$\xtext$};
		}
		\draw [red](3.5,0.4cm+3pt) -- (3.5,0.3);
		\node at (3.5,0.2) [anchor=north,red] {$\thres$};
		
		\draw [decorate,decoration={brace},xshift=0pt,yshift=1.3cm]
		(1,0) -- (3.5,0)node [black,midway,yshift=0.4cm] {
			$\labelgap{1}$
		};

		\draw [decorate,decoration={brace},xshift=0pt,yshift=0.7cm]
		(2.7,0) -- (3.5,0)node [black,midway,yshift=0.4cm] {
			$\labelgap{3}$
		};
		\draw [decorate,decoration={brace},xshift=0pt,yshift=0.7cm]
		(3.5,0) -- (5.1,0)node [black,midway,yshift=0.4cm] {
			$\labelgap{5}$
		};        
		\draw [decorate,decoration={brace},xshift=0pt,yshift=1.3cm]
		(3.5,0) -- (4.0,0)node [black,midway,yshift=0.4cm] {
			$\labelgap{4}$
		};
		
		\end{tikzpicture}
	\end{subfigure}
	\hfill
	\begin{subfigure}[b]{0.48 \textwidth}
		\centering
		\begin{tikzpicture}[>=latex,scale=1.0,baseline={(0.5,0)}]
		\def\recw{0.2}
		\draw [rounded corners,orange] (2.7+\recw, -\recw)--(2.7+\recw,+\recw)--(1-\recw, +\recw)--(1-\recw,-\recw)--cycle;
		\draw [rounded corners,cyan] (5.1+\recw, -\recw)--(5.1+\recw,+\recw)--(4.0-\recw, +\recw)--(4.0-\recw,-\recw)--cycle;    
		\draw (0.4,0.4) -- (5.8,0.4); 
		\foreach \x/\xtext in {1/\pbbeat{1}, 2/\pbbeat{2},
			2.7/\pbbeat{3}, 4.0/\pbbeat{4}, 5.1/\pbbeat{5}}{ 
			\draw (\x,0.4cm+3pt) -- (\x,0.3);
			\node at (\x,0.2) [anchor=north] {$\xtext$};
		}
		
		\draw [decorate,decoration={brace},xshift=0pt,yshift=1.3cm]
		(1,0) -- (2.7,0)node [black,midway,yshift=0.4cm] {
			$\compgap{1}$
		};
		\draw [xshift=0pt,yshift=-0.7cm,red]
		(2.7,0) -- (4.0,0)node [black,midway,yshift=-0.2cm] {
			(2)
		};
		\draw [red](2.7,-0.8) -- (2.7,-0.6);
		\draw [red](4.0,-0.8) -- (4.0,-0.6);
		\draw [xshift=0pt,yshift=-0.7cm,red]
		(1.0,0) -- (2.7,0)node [black,midway,yshift=-0.2cm] {
			(1)
		};
		\draw [red](1.0,-0.8) -- (1.0,-0.6);
		\draw [xshift=0pt,yshift=-0.7cm,red]
		(4.0,0) -- (5.1,0)node [black,midway,yshift=-0.2cm] {
			(5)
		};
		\draw [red](5.1,-0.8) -- (5.1,-0.6);
		\draw [xshift=0pt,yshift=-1.2cm,red]
		(2,0) -- (4.0,0)node [black,midway,yshift=-0.2cm] {
			(4)
		};
		\draw [red](2,-1.3) -- (2,-1.1);
		\draw [red](4.0,-1.3) -- (4.0,-1.1);
		\draw [xshift=0pt,yshift=-1.2cm,red]
		(1.0,0) -- (2,0)node [black,midway,yshift=-0.2cm] {
			(3)
		};
		\draw [red](1.0,-1.3) -- (1.0,-1.1);

		\draw [decorate,decoration={brace},xshift=0pt,yshift=0.7cm]
		(2.7,0) -- (4.0,0)node [black,midway,yshift=0.4cm] {
			$\diffcomp_1$
		};
		
		\draw [decorate,decoration={brace},xshift=0pt,yshift=1.3cm]
		(4.0,0) -- (5.1,0)node [black,midway,yshift=0.4cm] {
			$\compgap{5},\diffcomp_5$
		};
		
		\end{tikzpicture}
	\end{subfigure}
	\caption{Graphical illustration of the quantities $\labelgap{\samiter}$ (left) and $\compgap{\samiter},\diffcomp_{\samiter}$ (right) for $\narm=5$ arms, with $\posset=\{4,5\}$. We have $\upthres=4$ and $\lowthres=3$. $\diffcomp_{1}$ is equal to the max of $\min\{(1),(2)\}$ and $\min\{(3),(4)\}$; $\diffcomp_{5}$ is equal to $\min\{(2),(5)\}$. \label{fig:graph_gaps} \msp}
	\lsp
\end{figure*}
We define two problem complexities w.r.t pulls and duels separately. 

\textbf{Pull complexity.} Following previous works on TBP and pure-exploration bandits \citep{chen2014combinatorial,locatelli2016optimal}, we introduce the following quantity to characterize the intrinsic problem complexity with direct pulls. In particular, recall that  $\labelgap{\samiter}=|\sammean{\samiter}-\thres|$ is the gap between arm $\samiter$ and threshold. Then the pull complexity is defined as 
\(
\dlcpl=\sum_{\samiter=1}^{\narm}\frac{1}{(\labelgap{\samiter})^2}. 
\)
Chen et al. \citep{chen2014combinatorial} shows that there exists an algorithm using at most $O(\dlcpl\log(\narm\dlcpl/\tol))$ pulls. Moreover, they show a lower bound that any pull-only algorithm would require at least $\Omega(\dlcpl\log(1/\tol))$ pulls to give correct output with probability $1-\tol$. We add another notation for a ``partial'' label complexity: let $\dlcpl(m)$ be the sum of the largest $m$ terms in $\dlcpl$. Namely, we sort $\mu_1,\ldots, \mu_K$ by their gap with threshold, i.e.,  $\labelgap{\bs_1}\leq \labelgap{\bs_2}\leq \cdots \leq \labelgap{\bs_\narm}$ (cf. Figure \ref{fig:graph_gaps} left), and $\dlcpl(m)=\sum_{j=1}^{m}  \frac{1}{(\labelgap{\bs_{j}})^2}$.

\textbf{Duel complexity.}  Now we define the complexity w.r.t. duels. Our goal is to use duels to infer the (positive or negative) label of arms without actually pulling them. Therefore the difficulty of inferring a positive arm $\samiter\in \posset$ will depend on its difference with the ``worst'' positive arm, and similarly $\samiter\in \negset$ with the ``best'' negative arm. Let $\lowthres=\argmax_{\samiter\in \negset} \pbbeat{\samiter}$ be the best negative arm and $\upthres=\argmin_{\samiter\in \posset} \pbbeat{\samiter}$ be the worst positive arm, where $p_i$ is defined in \eqref{def:p}. And for any arm $\samiter\in \posset$, let $\compgap{\samiter}=\pbbeat{\samiter}-\pbbeat{\upthres}$ be the gap with arm $\upthres$ and similarly for any arm $j\in \negset$  define $\compgap{j}=\pbbeat{\lowthres}-\pbbeat{j}$. Intuitively, the complexity of identifying arm $\samiter$ through duels should depend on $\compgap{\samiter}$, and we therefore define $\compcpllow=\sum_{\samiter=1}^\narm \frac{1}{(\compgap{\samiter})^2}$.

Moreover,  it is worthwhile noting that the complexity of inferring a positive arm $\samiter$ using arm $\upthres$ will not only depend on $\pbbeat{\samiter}-\pbbeat{\upthres}$, but also on $\pbbeat{\upthres}-\pbbeat{\lowthres}$. If the gap $\pbbeat{\upthres}-\pbbeat{\lowthres}$ is very small, we cannot easily differentiate $\upthres$ from the other negative arms. On the other hand, we can use any positive arm $\samiterp$ to infer about arm $\samiter$, when $\pbbeat{\upthres}\leq \pbbeat{\samiterp}<\pbbeat{\samiter}$. To this end, we define
\[\diffcomp_\samiter=\begin{cases}
\displaystyle \max_{j\in \posset} \min\{\pbbeat{j}-\pbbeat{\lowthres}, \pbbeat{\samiter}-\pbbeat{j} \} & \text{ if } \samiter \in \posset,\\
\displaystyle \max_{j\in \negset} \min\{\pbbeat{j}-\pbbeat{\samiter}, \pbbeat{\upthres}-\pbbeat{j} \}, & \text{ if } \samiter \in \negset,
\end{cases} \]

See Figure \ref{fig:graph_gaps} right for a reference. And we define another duel complexity as 
\(\compcpl=\sum_{\samiter\in \armset \setminus \{\upthres,\lowthres \}} \frac{1}{\displaystyle\left( \diffcomp_\samiter\right)^2}. \)

\textbf{Relation between $\compgap{\samiter}$ and $\diffcomp_{\samiter}$.} Although we always have $\compgap{\samiter}\geq \diffcomp_{\samiter}$ and thus $\compcpllow\leq \compcpl$, in many situations $\compgap{\samiter}$ and $\diffcomp_{\samiter}$ are of similar scales. To see this, notice that $\diffcomp_{\samiter}\geq \min\{\compgap{\samiter}, \pbbeat{\upthres}-\pbbeat{\lowthres} \}$. In many cases in practice, we would expect a gap between $\posset$ and $\negset$, and therefore $\pbbeat{\upthres}-\pbbeat{\lowthres}$ will be a constant. We give a formal proposition about the relation between $\compcpl$ and $\compcpllow$ under Massart noise condition in Section \ref{sec:example_upper}.
\newcommand{\linkfunc}{\sigma}
\newcommand{\massartnoise}{c}

In Section \ref{sec:algo}, we present an upper bound using $\compcpl$, and in Section \ref{sec:lower}, we present a lower bound using $\compcpllow$. 
\msp
\section{The Rank-Search (\algoname) Algorithm \label{sec:algo}}
\ssp

We present our Rank-Search algorithm in this section. We give a detailed description of the algorithm in Section \ref{sec:algo_descr}, and analyze its theoretical performance in Section \ref{sec:upper_thm}.
\ssp
\subsection{Algorithm Description\label{sec:algo_descr}}
\ssp
\newcommand{\labeledsetBS}{T}
\newcommand{\shrinkfactor}{\kappa}

\begin{algorithm}[htb!]
	\caption{Rank-Search (\algoname)}
	\label{algo:ranksearch}
	\begin{algorithmic}[1]        
		\Require{Set of arms $\armset$, noise tolerance $\tol$, threshold $\thres$, initial confidence level $\comperr_0$, shrinking factor $\shrinkfactor$}
		\State $\workset\leftarrow \armset=[\narm]$, counter $\itercount\leftarrow 0$
		\State For every $\bs\in \workset$, let $\ncomp{\bs}\leftarrow 0, \nwin{\bs}\leftarrow 0$ \\
		\Comment{$\ncomp{\bs}$: Comparison count, $\nwin{\bs}$: Win count}
		\While{$\workset \ne \emptyset$}
		\While{$\exists \bs\in \workset, \ncomp{\bs}\leq \frac{1}{\comperr_\itercount^2} \log\left(\frac{8|S|(\itercount+1)^2}{\tol}\right) $} \label{step:rank_start}
		\For{$\bs\in \workset$}
		\State Draw $\bs'\in [\narm]$ uniformly at random, and compare arm $\bs$ with arm ${\bs'}$
		\State If arm $\bs$ wins, $\nwin{\bs}\leftarrow \nwin{\bs}+1$
		\State $\ncomp{\bs}\leftarrow \ncomp{\bs}+1$
		\EndFor
		\EndWhile
		\State Compute $\estprob{\bs}\leftarrow \nwin{\bs}/\ncomp{\bs}$ for all $\bs\in \workset$ 
		\State Rank arms in $\workset$ according to their $\estprob{i}$: $\workset=(\bs_1,\bs_2,...,\bs_{|\workset|}), \estprob{\bs_1}\leq \estprob{\bs_2}\leq\cdots\leq \estprob{\bs_{|\workset|}}$ \label{step:rank}
		\State Get $(\bsvar, \labeledsetBS) = \bsalgoname(S, \thres, \tol/4(t+1)^2)$ \label{step:binary_search} 
		\State If $\bsvar<|S|$, let $\overline{\workset}=\{\bs\in \workset: \estprob{\bs}-\estprob{{\bs_{\bsvar+1}}}> 2\comperr_{\itercount} \}$; for $\bs\in \overline{\workset}$, set $\hat{\dlabel}_\bs=1$ \label{step:infer_pos}
		\State If $\bsvar>0$, let $\underline{\workset}=\{\bs\in \workset: \estprob{\bs}-\estprob{\bs_{\bsvar}}< -2\comperr_{\itercount} \}$; for $\bs\in \underline{\workset}$, set $\hat{\dlabel}_\bs=0$ \label{step:infer_neg}        
		\State $\workset\leftarrow \workset-\overline{\workset}-\underline{\workset} - \labeledsetBS$
		\State $\comperr_{\itercount+1} \leftarrow \comperr_{\itercount}/\shrinkfactor$
		\State $\itercount\leftarrow \itercount+1$
		\EndWhile        
		\Ensure{$\estposset=\{\samiter:\hat{\dlabel}_{\samiter}=1\}, \estnegset = \armset \setminus \estposset$}
	\end{algorithmic}
\end{algorithm}
Algorithm \ref{algo:ranksearch} describes the Rank-Search algorithm.
At a high level, \algoname alternates between ranking items using duels (Line \ref{step:rank_start}-\ref{step:rank}), and a binary search using pulls (Line \ref{step:binary_search} and Algorithm \ref{algo:binary_search}). We first initialize the work set $\workset$ with all arms, and comparison confidence $\comperr_0=1/4$. In the rank phase, we iteratively compare each arm $\samiter\in \workset$ with a random arm, as an unbiased estimator for $\pbbeat{\samiter}$. After each arm has received $\frac{\log(2/\tol_t)}{\comperr_{\itercount}^2}$ comparisons, we rank the arms in $\workset$ according to their win rates $\estprob{\samiter}$. Then Algorithm \ref{algo:binary_search} performs binary search on the sorted $\workset$ to find the boundary between positive and negative arms (detailed below). 

Our binary search is a standard process: it starts with the middle of the sequence, and if the middle arm is positive, we move to the first half (i.e., arms with smaller estimated means), and otherwise, we move to the second half (i.e., arms with larger estimated means). Algorithm 2 just behaves as if $S$ is perfectly ranked. It is worthwhile noting that since $S$ is not ranked according to the real $\pbbeat{\samiter}$'s, there might be negative samples larger than positive samples in $\workset$. However, we show that \algoname can still run effectively even with a misranked $\workset$.
We figure out the label of the middle point using \figurelabelname (Algorithm \ref{algo:figure_out_label}). \figurelabelname aims to solve the simple TBP in the one-arm setting: We keep a confidence interval $\estmean{\samiter}\pm \gamma$ in each round and return the label once $\thres$ is not in the interval. 

\bsalgoname returns the boundary $\bsvar$. Let $\overline{\workset}=\{\bs\in \workset: \estprob{\bs}-\estprob{{\bs_{\bsvar+1}}}> 2\comperr_{\itercount} \}$ be the arms that are separated from arm $\bs_{\bsvar+1}$, and we label $\samiter\in \overline{\workset}$ as positive; we do similarly for negative arms. Then we update working set $\workset$ with all the unlabeled arms, and we shrink the confidence level by a constant factor $\shrinkfactor>1$.

\begin{algorithm}[htb!]
	\caption{\bsalgoname}
	\label{algo:binary_search}
	\begin{algorithmic}[1]        
		\Require{Sequence $\workset=(\bs_1,\bs_2,...,\bs_{|\workset|})$, threshold $\thres$, confidence $\tol_0$}
		\State $\bsvar_{\min}\leftarrow 0, \bsvar_{\max} \leftarrow |\workset|, \labeledsetBS = \emptyset$
		\While{$\bsvar_{\min}<\bsvar_{\max}$}
		\State $\bsvar=\lceil(\bsvar_{\min}+\bsvar_{\max})/2\rceil$
		\State $\hat{\dlabel}_{\bs_\bsvar}=\figurelabelname(\bs_\bsvar,\thres, \frac{\tol}{\log |S|})$ 
		\State $\labeledsetBS = \labeledsetBS \cup \{\bs_{\bsvar} \}$
		\If{$\hat{\dlabel}_{\bs_\bsvar}=1$}
		\State $\bsvar_{\max}=\bsvar-1$ \label{step:moveleft}
		\Else
		\State $\bsvar_{\min}=\bsvar$\label{step:moveright}
		\EndIf
		\EndWhile
		\Ensure{Boundary $\bsvar_{\min}$, labeled arms $\labeledsetBS$}
	\end{algorithmic}
\end{algorithm}

\begin{algorithm}[htb!]
	\caption{\figurelabelname}
	\label{algo:figure_out_label}
	\begin{algorithmic}[1]        
		\Require{Arm $\bs$, threshold $\thres$, confidence $\tol_1$}
		\State $t\leftarrow 0$ 
		\State Define $\nlabel{\bs}\leftarrow 0, \bandsum{\bs}\leftarrow 0$
		\Repeat
		\While{$\nlabel{\bs}\leq 2^t$}
		\State Query $\dlabelrv_\bs$, and let $\bandsum{\bs}\leftarrow \bandsum{\bs}+\dlabelrv_\bs, \nlabel{\bs}\leftarrow
		\nlabel{\bs}+1$
		\EndWhile
		\State Compute $\estmean{\bs}\leftarrow \bandsum{\bs}/\nlabel{\bs}$
		\State $\gamma = \sgpara\sqrt{\frac{2\log(4(t+1)^2/\tol_1)}{\nlabel{\bs}}}$  \label{step:delta0} 
		\State $t\leftarrow t+1$
		\Until{$|\estmean{\bs}-\thres|> \gamma$}
		\Ensure{Predicted label $\hat{y}_{\bs}=I(\estmean{\bs}>\thres)$}
	\end{algorithmic}
\end{algorithm}
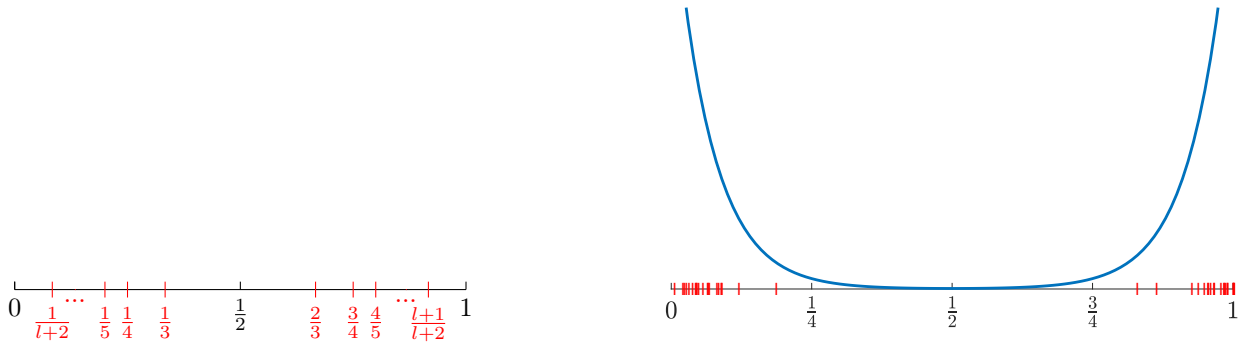
\begin{figure*}[th!]
	\centering
	\begin{subfigure}[b]{0.45\textwidth}\label{fig:ex1}
		\centering
		\begin{tikzpicture}\label{ex3}
		\draw[-] (-3,0) -- (3,0) ;
		\draw[shift={(-3,0)},color=black](0pt,3pt) -- (0pt,0pt);
		\draw[shift={(-3,0)},color=black] (0pt,3pt) -- (0pt,0pt) node[below] 
		{0};
		\draw[shift={(3,0)},color=black] (0pt,3pt) -- (0pt,0pt) node[below] 
		{1};
		\draw[shift={(0,0)},color=black] (0pt,3pt) -- (0pt,0pt) node[below] 
		{$\frac{1}{2}$};
		\draw[shift={(-1.0,0)},color=red] (0pt,3pt) -- (0pt,-3pt) node[below] 
		{$\frac{1}{3}$};
		\draw[shift={(-1.5,0)},color=red] (0pt,3pt) -- (0pt,-3pt) node[below] 
		{$\frac{1}{4}$};
		\draw[shift={(-1.8,0)},color=red] (0pt,3pt) -- (0pt,-3pt) node[below] 
		{$\frac{1}{5}$};
		\draw[shift={(-2.2,0)},color=red] (0pt,0pt) -- (0pt,0pt) node[below] 
		{...};
		\draw[shift={(-2.5,0)},color=red] (0pt,3pt) -- (0pt,-3pt) node[below] 
		{$\frac{1}{l+2}$};
		\draw[shift={(1.0,0)},color=red] (0pt,3pt) -- (0pt,-3pt) node[below] 
		{$\frac{2}{3}$};
		\draw[shift={(1.5,0)},color=red] (0pt,3pt) -- (0pt,-3pt) node[below] 
		{$\frac{3}{4}$};
		\draw[shift={(1.8,0)},color=red] (0pt,3pt) -- (0pt,-3pt) node[below] 
		{$\frac{4}{5}$};
		\draw[shift={(2.2,0)},color=red] (0pt,0pt) -- (0pt,0pt) node[below] 
		{...};
		\draw[shift={(2.5,0)},color=red] (0pt,3pt) -- (0pt,-3pt) node[below] 
		{$\frac{l+1}{l+2}$};
		\end{tikzpicture}
	\end{subfigure}
	\hfill
	\begin{subfigure}[b]{0.45\textwidth}\label{fig:ex2}
		\centering
		\includegraphics[width=\textwidth]{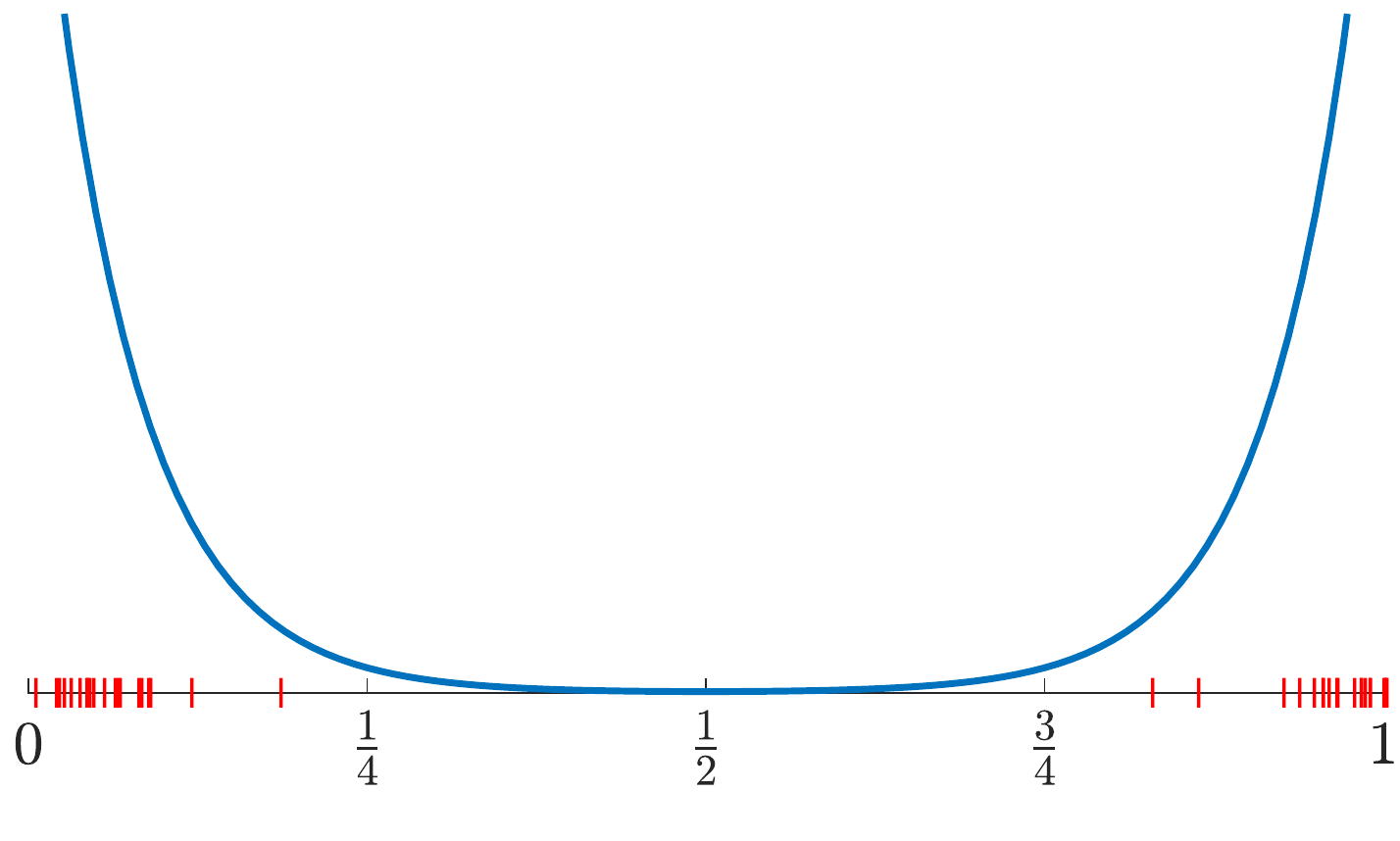}
	\end{subfigure}
	
	\caption[Graphical description]{\label{fig:examples}Graphical illustration of the examples. Each red vertical line corresponds to one arm $\samiter$, and $\thres=1/2$. Left: Example 1 with fixed means. Right: Example 2 with $\narm=40$ arms. The blue curve illustrates the pdf of all arm means.}
	
\end{figure*}

\ssp
\subsection{Theoretical Analysis \label{sec:upper_thm}}
\ssp

\newcommand{\ndirectsample}{n_l}
\newcommand{\lowestgaplabel}{\Delta^*}
We now present the theorem about performance of \algoname.
\begin{theorem}
	\label{thm:upper}
	Let $\comperr^*=\min_{\samiter \in \armset\setminus \{\upthres,\lowthres\}} \diffcomp_{\samiter}$ and $\lowestgaplabel=\min_{\samiter} \labelgap{\samiter}$. 
	Then with probability $1-\tol$ \algoname succeeds, and the number of duels and pulls it uses are bounded by 
	\begin{align*}
	n_{\text{duel}}&\leq 32\compcpl\log \frac{4\narm\log(1/\comperr^*) }{\tol},\\
	n_{\text{pull}}&\leq 16R^2\dlcpl(\ndirectsample)\log\left(\frac{ \ndirectsample\log (1/\lowestgaplabel)}{\tol}\right),
	\end{align*}
	where $\ndirectsample$ is the number of times \figurelabelname is called, and we have $\ndirectsample= O(\log \narm\log (1/\comperr^*))$.
\end{theorem}

\textbf{Remark.} First, the results in \citep{chen2014combinatorial} correspond to using $O\left(\dlcpl(\narm)\log\left(\frac{\dlcpl(\narm)\narm}{\tol}\right)\right)$ pulls to get $\tol$ confidence. In terms of number of direct pulls, \algoname can reduce $\narm$ dependence to $\log\narm$ dependence when $\comperr^*$ is a constant, an exponential improvement.\\
Second, in terms of number of duels, \algoname has a requirement based on dueling complexity $\compcpl$ instead of $\dlcpl$. In many cases, $\compcpl$ is close to $\dlcpl$, and we point out several such cases in Section \ref{sec:example_upper}. 
Thus, we see that in the Dueling-choice framework, the number of pulls required improves exponentially in dependence on $K$ at the expense of requiring a number of duels proportional to number of pulls in pull-only case. 


\msp
\section{Implications of Upper Bounds in Special Cases \label{sec:example_upper}}
\ssp


We provide two examples to compare our theoretical upper bounds with the classical pull-only TBP. Throughout this section, we will assume that all the observations follow Bernoulli distributions, and $\thres=1/2$. The examples we raise in this section all follow the Massart noise condition, i.e., $|\sammean{\samiter}-\thres|\geq c$ that is well known in classification analysis \citep{2007math......2683M}. We first give the following proposition to show that \algoname is optimal under Massart noise.

\ssp
\begin{prop}\label{prop:massart}
	Suppose $\labelgap{\samiter}\geq \massartnoise$ for some $\massartnoise$ for all arm $\samiter$, and $\prefmat_{\samiter\samiterp}=\frac{1}{2}+\linkfunc(\sammean{\samiter}-\sammean{\samiterp})$ for some increasing link function $\linkfunc: \mathbb{R}\rightarrow [-1/2, 1/2]$. Also assume for any $x,y \in [\sammean{1},\sammean{\narm}]$ we have $\linkfunc(x-y)\geq L(x-y)$ for some constant $L$. Then we have i) $\pbbeat{\upthres}-\pbbeat{\lowthres}\geq 2L\massartnoise$, ii) $\diffcomp_{\samiter}\geq \min\{2L\massartnoise,\compgap{\samiter} \}$, and iii) $\compcpl\leq \frac{1}{4L^2\massartnoise^2}\compcpllow$. 
\end{prop}
\ssp

Proposition \ref{prop:massart} shows that $\compcpl=O(\compcpllow)$ under Massart noise and the assumption that a link function exists. The assumption of such a link function is satisfied by several popular comparison models including the Bradley-Terry-Luce (BTL) \citep{bradley1952rank} and Thurstone models \citep{thurstone1927law}.
We now give two positive examples that \algoname will lead to a gain compared with the pull-only setting.
For simplicity we will suppose duels follow a comparison model given as follows:  $M_{ij} = \Pr[{\samiter}\succ {\samiterp}]=\frac{1+\sammean{\samiter}-\sammean{\samiterp}}{2}$. This is known as the linear link function since it linearly relates the duel win probability with the reward means. Routine calculations show that under a linear link function we have $\pbbeat{\samiter}-\pbbeat{\samiterp}=\Theta(\sammean{\samiter}-\sammean{\samiterp})$. We require that both our method and pull-only method succeed with probability $1-\tol$, with a small constant $\tol$ (e.g., $\tol=0.05$). Both of our positive examples assume that the means are dense near the boundaries given by $\mu=0$ and $\mu=1$, while a very small fraction of arms have means near 1/2, so that there is a significant gap between the arms $\upthres$ and $\lowthres$ closest to the threshold, as well as any arm $\samiter$ and arm $\upthres$ or $\lowthres$ that is closest to it(cf. Figure \ref{fig:examples}). Although these examples can look artificial at first sight, we note that such a bowl-shaped distribution is common in practice, and is similar to Tsybakov noise \citep{tsybakov2004optimal} assumption used to characterize classification noise in the machine learning literature.

\noindent\textbf{Example 1.} Suppose $\narm=2l$, and $\sammean{\samiter}=\frac{1}{(l+3)-\samiter}$ for $1\leq \samiter\leq l$, and $\sammean{\samiter}=1-\frac{1}{\samiter-(l-2)}$ for $l+1\leq \samiter\leq 2l$ (see Figure \ref{fig:examples} left). We will have $\labelgap{\samiter}=\diffcomp_{\samiter}=\Omega(1)$ for all arms $\samiter\in \armset$. Then the previous state-of-art CLUCB algorithm requires $O(\narm\log \narm)$ pulls, and their lower bounds show that any pull-only algorithm needs at least $\Omega(\narm)$ pulls. On the other hand, our algorithm requires $O(\narm\log \narm)$ duels and only $O(\log^2 \narm)$ pulls. When pulls are more expensive than duels, there is a significant cost saving when using our RS algorithm.

%

\noindent\textbf{Example 2.} Suppose $\narm=2l$. Sample $x_1,...,x_{\narm}$ from an exponential distribution with parameter $\lambda=4\log(4l/\tol)$, and let $\sammean{\samiter}=x_\samiter$ for $1\leq \samiter\leq l$, and $\sammean{\samiter}=1-x_{\samiter}$ for $l+1\leq \samiter\leq 2l$ (see Figure \ref{fig:examples} right). Then with probability $1-\tol$: i) $\sammean{\samiter}\in [0,1] \; \forall \samiter\in [\narm]$; ii)  $\labelgap{\samiter}=\Omega(1)$, and $\compcpl=\compcpllow$; iii) CLUCB requires $O(\narm \log \narm)$ pulls, and any pull-only algorithm requires at least $\Omega(\narm)$ pulls; iv) Our algorithm requires $O(\narm \log^3\narm)$ duels and $O(\log^2 \narm)$ pulls.

We provide proofs of the results for these two examples in the appendix.

%
%

\msp
\section{Lower Bounds \label{sec:lower}}
\ssp

In this section, we give lower bounds that complement our upper bounds. We first give an arm-wise lower bound in Section \ref{sec:lower_arm} to show that \algoname is almost optimal in terms of the total number of queries to each individual arm. Then, we discuss the optimality of both $\numduel$ and $\numpull$ in Section \ref{sec:lower_all}.

For simplicity, in this section we suppose all rewards follow a Gaussian distribution with parameter $R$, i.e., $\samdistr{\samiter}=\mathcal{N}(\sammean{\samiter},R^2)$.
Our results can be easily extended to other sub-Gaussian distributions (e.g., when all rewards are binary). 
\ssp
\subsection{An Arm-Wise Lower Bound \label{sec:lower_arm}}
\ssp

The following theorem gives a lower bound on the number of pulls and duels on a particular arm $\samiterint$.
\ssp
\begin{theorem}\label{thm:lower}
	Suppose an algorithm $\algo$ recovers $\posset$ with probability $1-\tol$ for any problem instance $(\prefmat,\meanvec)$ and $\tol\leq 0.15$. For any arm $\samiter$, let $\dueltime{\samiter}^{\algo}$ be the number of times that arm $\samiter$ is selected for a duel, and $\pulltime{\samiter}^{\algo}$ be the number of times that arm $\samiter$ is pulled. Let $c=\min\{\frac{1}{10}, \frac{\sgpara^2}{2}\}$. Then for any problem instance $(\prefmat,\meanvec)$ with $\prefmat_{\samiter\samiterp}\geq \frac{3}{8}$ for every arm $\samiter,\samiterp\in [\narm]$, and a specific arm $\samiterint\in \posset$, we have
	\begin{equation}
	\E_{\prefmat,\meanvec}[(\compgap{\samiterint})^2\dueltime{\samiterint}^{\algo}+(\labelgap{\samiterint})^2\pulltime{\samiterint}^{\algo}]\geq c\log(\frac{1}{2\tol}).\label{eqn:lower_armwise}
	\end{equation}
\end{theorem}
\ssp
Theorem \ref{thm:lower} shows an arm-wise lower bound that the sum of duels and pulls (weighted by their complexity) must satisfy (\ref{eqn:lower_armwise}). In the pull-only setting, this agrees with the known result that number of pulls needed for an arm $\samiterint$ is $\Omega((\labelgap{\samiterint})^{-2})$. And for duel-choice setting, it shows that if we never pull arm $\samiterint$, number of duels involving arm $\samiterint$ (against some known arm) is at least $\Omega((\compgap{\samiterint})^{-2})$. From our proof of Theorem \ref{thm:upper}, we can easily show the following proposition for the upper bound that \algoname achieves:
\begin{prop}\label{prop:lower}
	For any problem instance $(\prefmat,\meanvec)$ and arm $\samiterint$, Algorithm  \algoname succeeds with probability at least $1-\delta$ and there exists a constant $C$ such that the \algoname algorithm achieves that
	\begin{equation}
	\E_{\prefmat,\meanvec}[(\diffcomp_{\samiterint})^2\dueltime{\samiterint}^{\algoname}+(\labelgap{\samiterint})^2\pulltime{\samiterint}^{\algoname}]\leq C\log\left(\frac{\narm \log(\frac{K}{\comperr^*\lowestgaplabel}) }{\tol}\right).\label{eqn:upper_armwise}
	\end{equation}
	\[ \]
\end{prop}
Comparing (\ref{eqn:upper_armwise}) with (\ref{eqn:lower_armwise}), our \algoname algorithm is arm-wise optimal except for the difference of $\compgap{\samiterint}$ vs. $\diffcomp_{\samiterint}$, and the log factors. This shows that \algoname is near optimal in the sum $\E_{\prefmat,\meanvec}[(\compgap{\samiterint})^2\dueltime{\samiterint}^{\algo}+(\labelgap{\samiterint})^2\pulltime{\samiterint}^{\algo}]$.
\ssp
\subsection{Optimality of $\numduel$ and $\numpull$ \label{sec:lower_all}}
\ssp

In this subsection, we analyze the lower bound of \problemname under the case when duels are much cheaper than pulls. In this case, we would like to minimize the number of pulls, and then minimize the number of duels. Intuitively, \algoname algorithm is optimal in $\numpull$ as it uses roughly $O(\log \narm)$ pulls; this is necessary even if we know a perfect ranking of all arms (due to the complexity of binary search). We consider an extreme case, where we know the label of arm $\upthres$ and $\lowthres$ from pulls, and wish to infer all other labels using duels. The following corollary of Theorem \ref{thm:lower} shows a lower bound in this case:
\begin{corollary}\label{col:lower_lowpull}
	Suppose an algorithm $\algo$ is given that $\upthres\in \posset$ and $\lowthres\in \negset$, and uses only duels. Under the same assumption as in Theorem \ref{thm:lower}, the number of duels of $\algo$ is at least $\E[n_{\text{duel}}^\algo] \geq c\compcpllow\log(1/2\tol)$.
\end{corollary}

Combining Corollary \ref{col:lower_lowpull} with the fact that $O(\log \narm)$ is necessary for \problemname, we 
show that \algoname is near optimal in both $\numduel$ and $\numpull$.

\msp
\section{Experiments\label{sec:expr}}
\ssp
To verify our theoretical insights, we perform experiments on a series of settings to illustrate the efficacy of \algoname, on both synthetic and real-world data. For comparison, we include the state-of-art CLUCB in the pull-only setting, and several naive baselines.\\

\begin{figure*}[th!]
	\centering
	\begin{subfigure}[b]{0.33\textwidth}
		\centering
		\includegraphics[width=\textwidth]{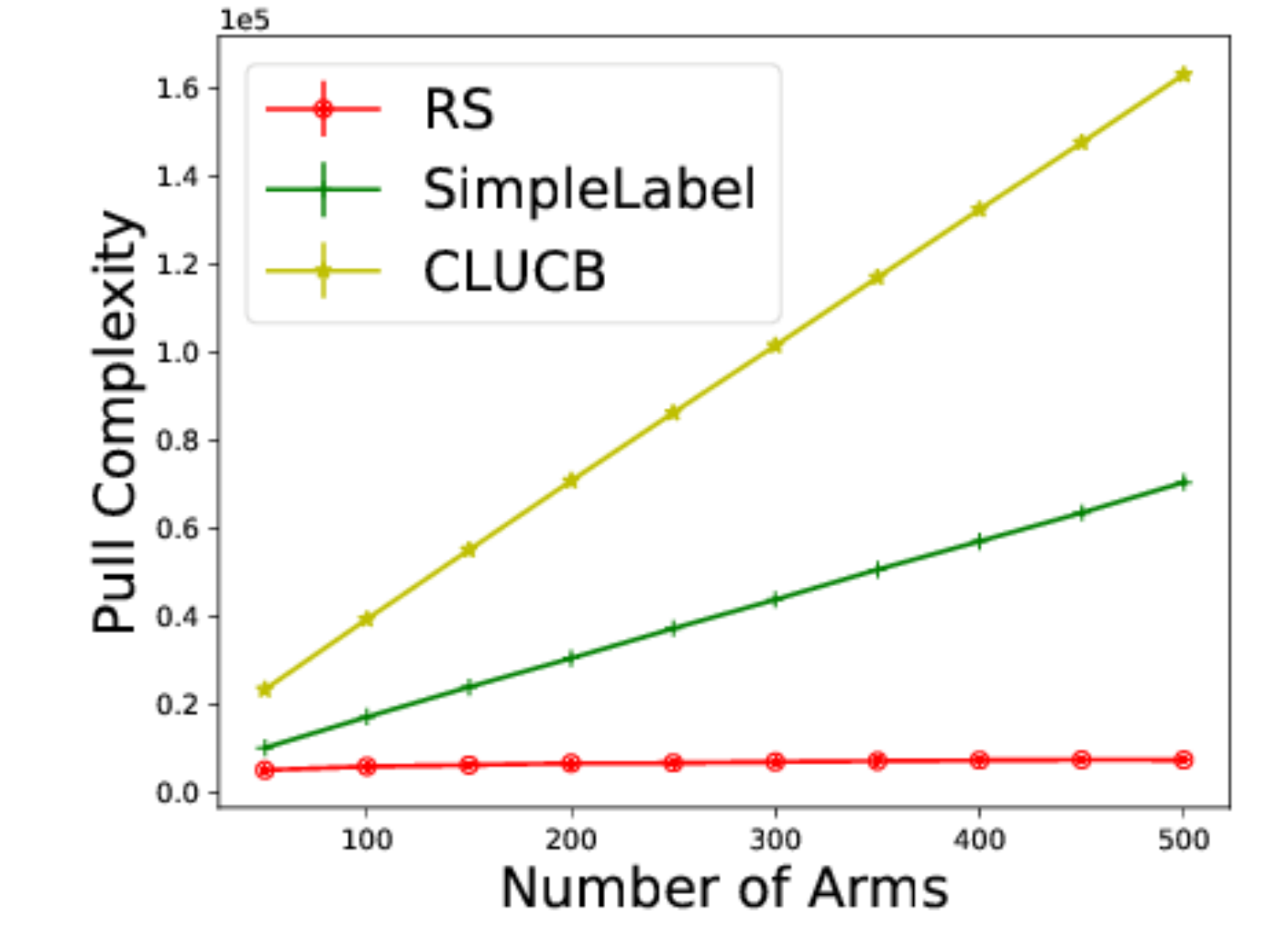}
		\caption{harmonic}
	\end{subfigure}%
	\begin{subfigure}[b]{0.33\textwidth}
		\centering
		\includegraphics[width=\textwidth]{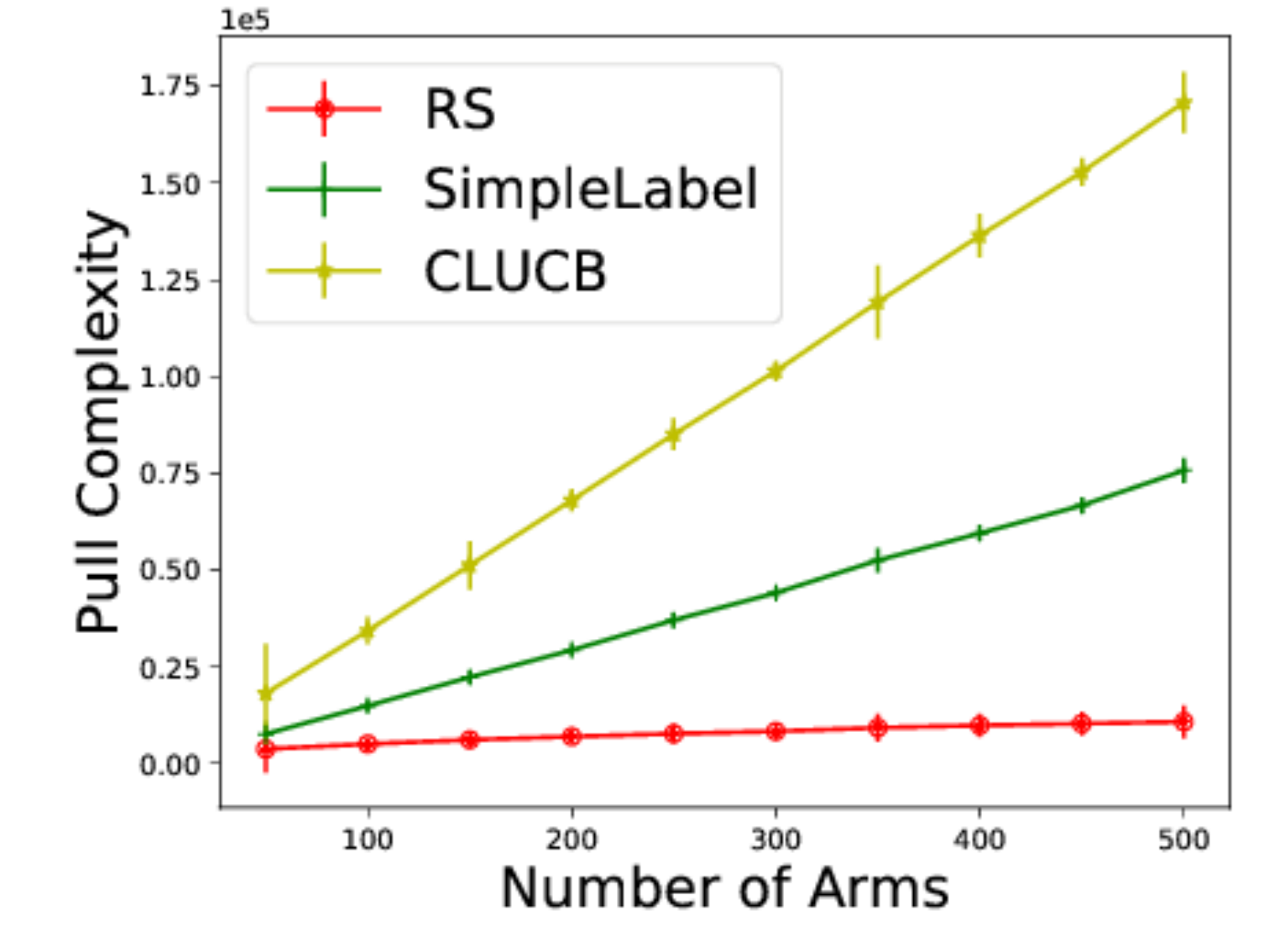}
		\caption{exponential}
	\end{subfigure}
	\begin{subfigure}[b]{0.33\textwidth}
		\centering
		\includegraphics[width=\textwidth]{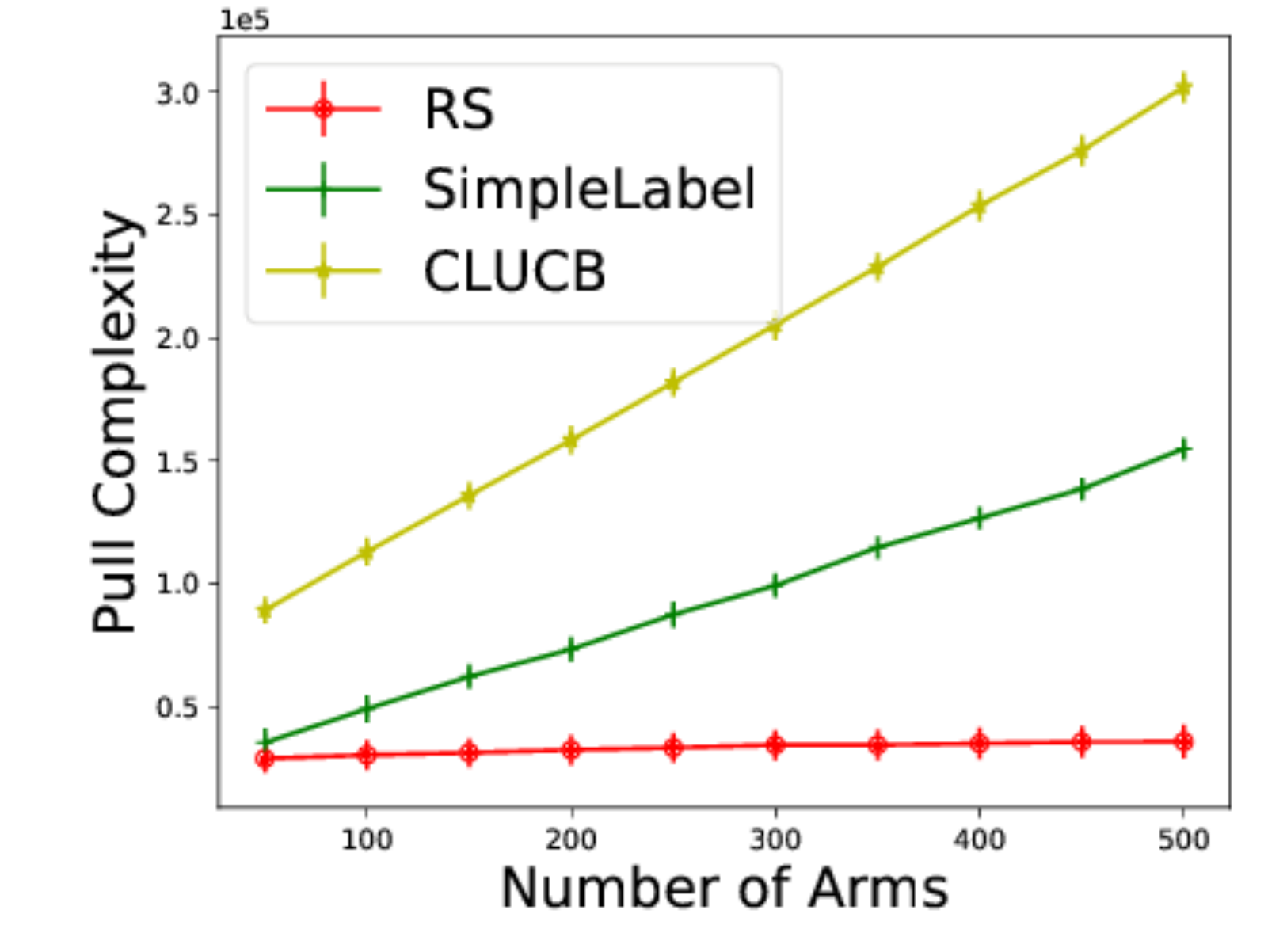}
		\caption{3 groups}
	\end{subfigure}
	\caption{Empirical results comparing \algoname and other baselines. Error bars represent standard deviation across 500 experiments.\label{fig:expr_result_syn}}
	\lsp
\end{figure*}

\subsection{Data Configuration}
\ssp

For synthetic data, we vary the number of arms $\narm$ from 50 to 500, and set threshold $\thres=0.5$. The duels follows from a linear link function $\Pr[\samiter\succ \samiterp]=\frac{1+\sammean{\samiter}-\sammean{\samiterp}}{2}$ \footnote{We include the results with a BTL model in the appendix.}. Let $\narm=2l$, the mean rewards are given as below:

\textbf{Experiment 1 (harmonic):} This is Example 1 from Section \ref{sec:example_upper}.\\
\textbf{Experiment 2 (exponential):} This is Example 2 from Section \ref{sec:example_upper}.\\ 
\textbf{Experiment 3 (3groups):} This is similar to the example in \citep{locatelli2016optimal}. Let $\sammean{i}=0.1$ for $i=1,2,...,l-2$, $\sammean{(l-1):(l+2)}=(0.35,0.45,0.55,0.65)$, and $\sammean{i}=0.9$ for $i=l+3,...,\narm$.\\

For real-world data, we use the reading difficulty dataset collected by \cite{chen2013pairwise}. The data consists of 491 passages, each with a reading difficulty level ranged in 1-12. We randomly take $\narm$ passages from the whole set, with $\narm$ varying from 50 to 491. Let $\mu_i=l_i/13$, where $l_i$ is the difficulty level of passage $i$. The goal here is to identify the difficult passages with level at least 7, or equivalently $\thres=0.5$. Although the original dataset from \cite{chen2013pairwise} comes with comparisons, it does not cover all pairs and we therefore use a probabilistic model to generate comparison feedbacks. Specifically, we experiment with two types of comparison models: i) linear link function $\Pr[\samiter\succ \samiterp]=\frac{1+\theta(\sammean{\samiter}-\sammean{\samiterp})}{2}$; ii) BTL model: \(\Pr[i\succ j]=\frac{1}{1+e^{(\sammean{j}-\sammean{i})\theta}}\). For both model, we find the $\theta$ that maximizes the log likelihood based on comparisons data provided in \citep{chen2013pairwise}. Hypothesis testing against a null hypothesis ($\Pr[i\succ j]=1/2$) gives $p$-values less than $1\times 10^{-4}$ for both models.


\subsection{Baselines and Implementation Details}
We compare performance of the following methods. 

\textbf{CLUCB}\citep{chen2014combinatorial}: We implement the CLUCB algorithm which only queries for selective direct pulls in a TBP setting. \\
\textbf{SimpleLabel}: This is a simple pull-only baseline where we apply \figurelabelname to all the arms $\samiter\in \armset$ with confidence $\tol/\narm$.\\
\textbf{RankThenSearch}: As we discussed in introduction, we compare to the baselines where we first use a ranking algorithm to rank all the arms, and then perform a binary search to find the boundary.
We consider two methods for the first ranking step. i) \textsf{ActiveRank}\citep{heckel2016active}: An active ranking algorithm that achieves optimal rates based on Borda scores. ii) \textsf{PLPAC-AMPR}\citep{szorenyi2015online}: Another ranking algorithm that focuses on BTL model. After the ranking algorithm we run a single binary search on the sorted sequence, using $\figurelabelname$ to identify labels. 
\\
\textbf{RankSearch:} Our algorithm. The parameters of our algorithms are the initial confidence $\comperr_0$ and shrinking factor $\shrinkfactor$. Both of them decides how aggressive we decrease our confidence: A small $\comperr_0$ will lead to a starting point with high confidence, and a large $\shrinkfactor$ will increase the confidence level quickly. Both of them will lead to a higher number of duels.
In our implementation, we pick $\comperr_0$ adaptively so that $\max \estprob{\bs}-\min \estprob{\bs}\geq 2\comperr_0$ (see Appendix for details), and use $\shrinkfactor=2$.

We note that previous works on TBP in the fixed budget setting \citep{locatelli2016optimal,mukherjee2017thresholding} cannot be implemented in our fixed-confidence setting.

We run all the methods
with varying number of arms, and compare their performance to reach confidence $\tol=0.95$.
For complexity notion, since there is no pre-defined cost ratio between duels and pulls, we compare the pull and duel complexity of \algoname separately with the baselines. Specifically, we compare pull complexity with SimpleLabel and CLUCB, and compare duel complexity with RankThenSearch (since the other two baselines are pull-only algorithms). Each experiment is repeated 500 times, and we compute the mean and standard deviation of each baseline's performance.
\begin{figure}[h!]
	\centering
	\begin{subfigure}[b]{0.35\textwidth}
		\centering
		\includegraphics[width=\textwidth]{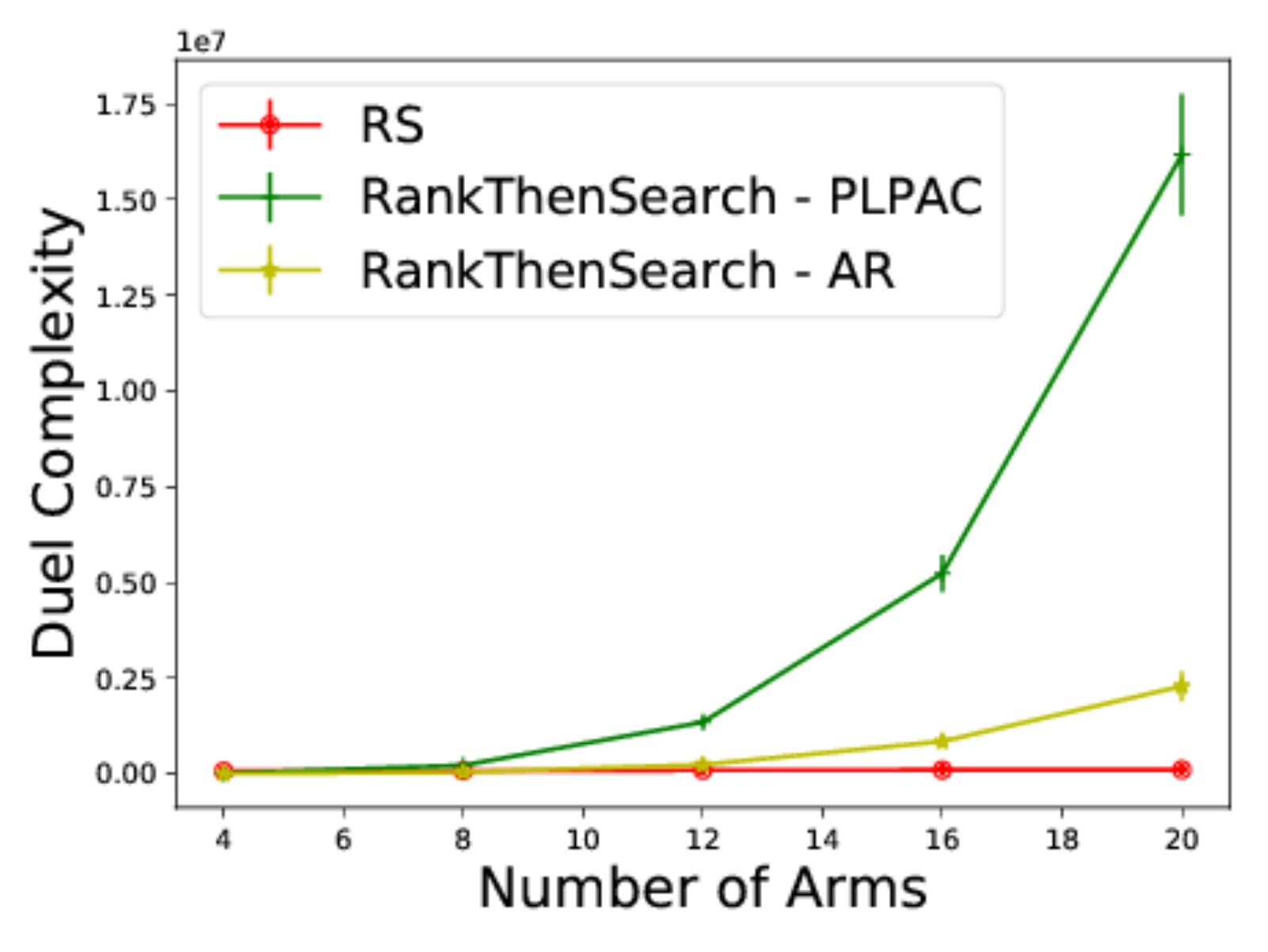}
		\caption{harmonic}
	\end{subfigure}%
	\hfill
	\begin{subfigure}[b]{0.35\textwidth}
		\centering
		\includegraphics[width=\textwidth]{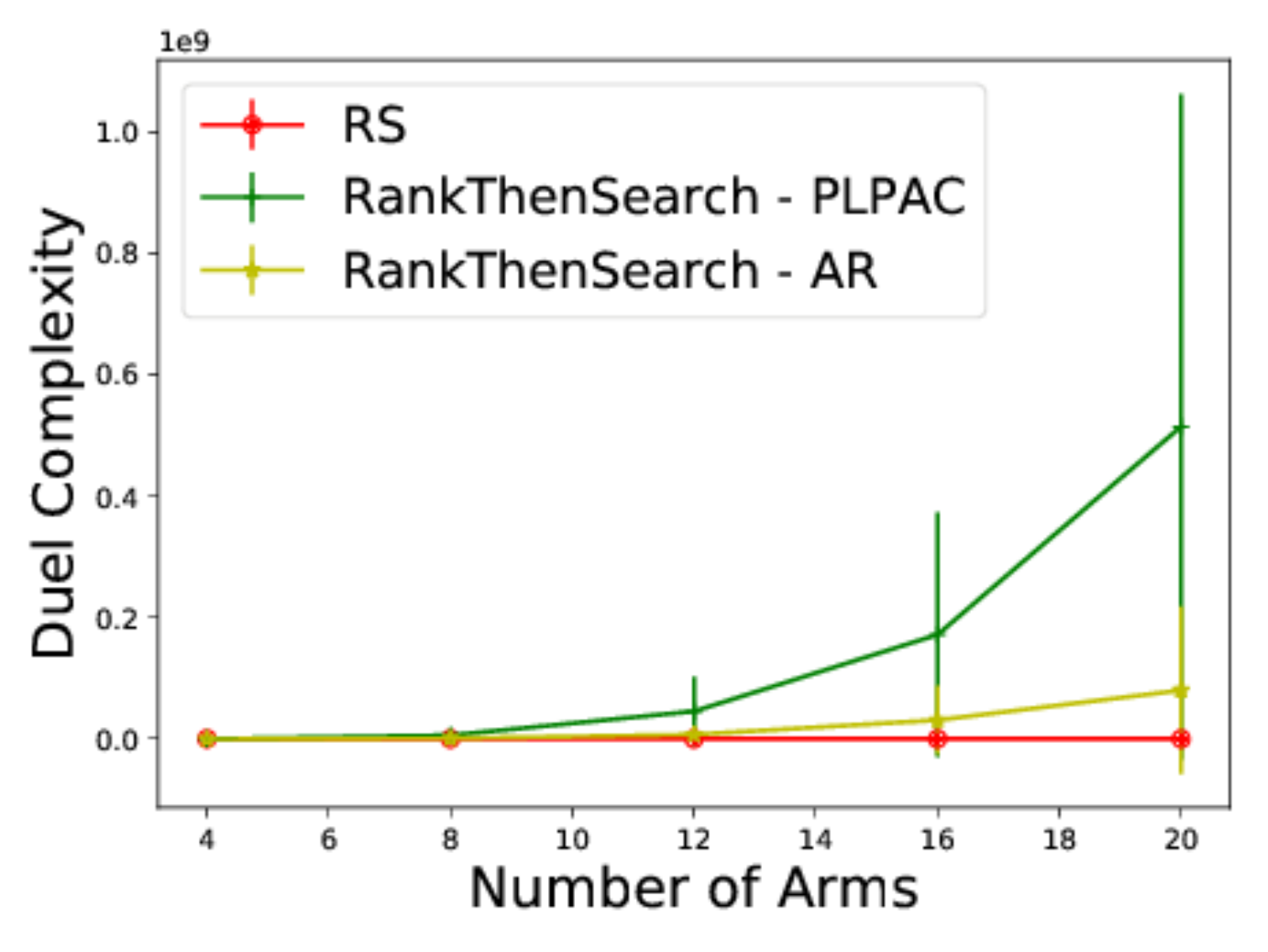}
		\caption{exponential}
	\end{subfigure}
	
	\caption{Empirical results comparing \algoname and RankThenSearch. Error bars represent standard deviation across 500 experiments. PLPAC is short for PLPAC-AMPR. \label{fig:expr_result_duelonly}}
\end{figure}

\ssp
\subsection{Experiment Results \label{sec:expr_res}}
\ssp

\textbf{Results on synthetic data.} In Figure \ref{fig:expr_result_syn}, we plot the empirical pull complexity of \algoname along with the baselines of CLUCB and SimpleLabel. As expected, the number of pulls of \algoname is much lower than the baseline algorithms in all three experiments we consider. Interestingly, SimpleLabel also has an advantage over CLUCB in the pull-only setting. We note that CLUCB's $O(\dlcpl\log(\frac{\dlcpl}{\tol}))$ is only optimal up to $\log(\dlcpl)$ factors, and SimpleLabel might have an advantage because its pull complexity is $O(\dlcpl\log(\frac{K\log\lowestgaplabel}{\tol}))$ in the pull-only setting, slightly better than CLUCB. This advantage and the optimal rate for the pull-only setting is of independent interest and we leave it as future work.

We then compare the duel complexity with RankThenSearch in Figure \ref{fig:expr_result_duelonly}. Since RankThenSearch needs to differentiate between every pair of arms, the algorithms take extremely long to run and we have to limit the arms to be at most 20 (as is done in \cite{szorenyi2015online}).
Note that since in \textbf{3groups} the arms are not separable, we only compare to RankThenSearch in the first two settings. The results show that RankThenSearch with ActiveRank and PLPAC-AMPR both acquires an incredible number of duels in order to rank the arms: to rank 20 arms they acquire hundreds of millions ($1\times 10^8$) of duels, for the \textbf{exponential} arm setup. This prohibitive cost makes it impossible to adopt the RankThenSearch method. We also observed a very large variance in performance for RankThenSearch, because differentiating arms close to each other is very unstable. Dueling complexity of \algoname is much lower and more stable than the above methods, and therefore \algoname achieves a balance between duels and pulls.

\begin{figure}[tbp!]
	\centering
	\begin{subfigure}[b]{0.35\textwidth}
		\centering
		\includegraphics[width=\textwidth]{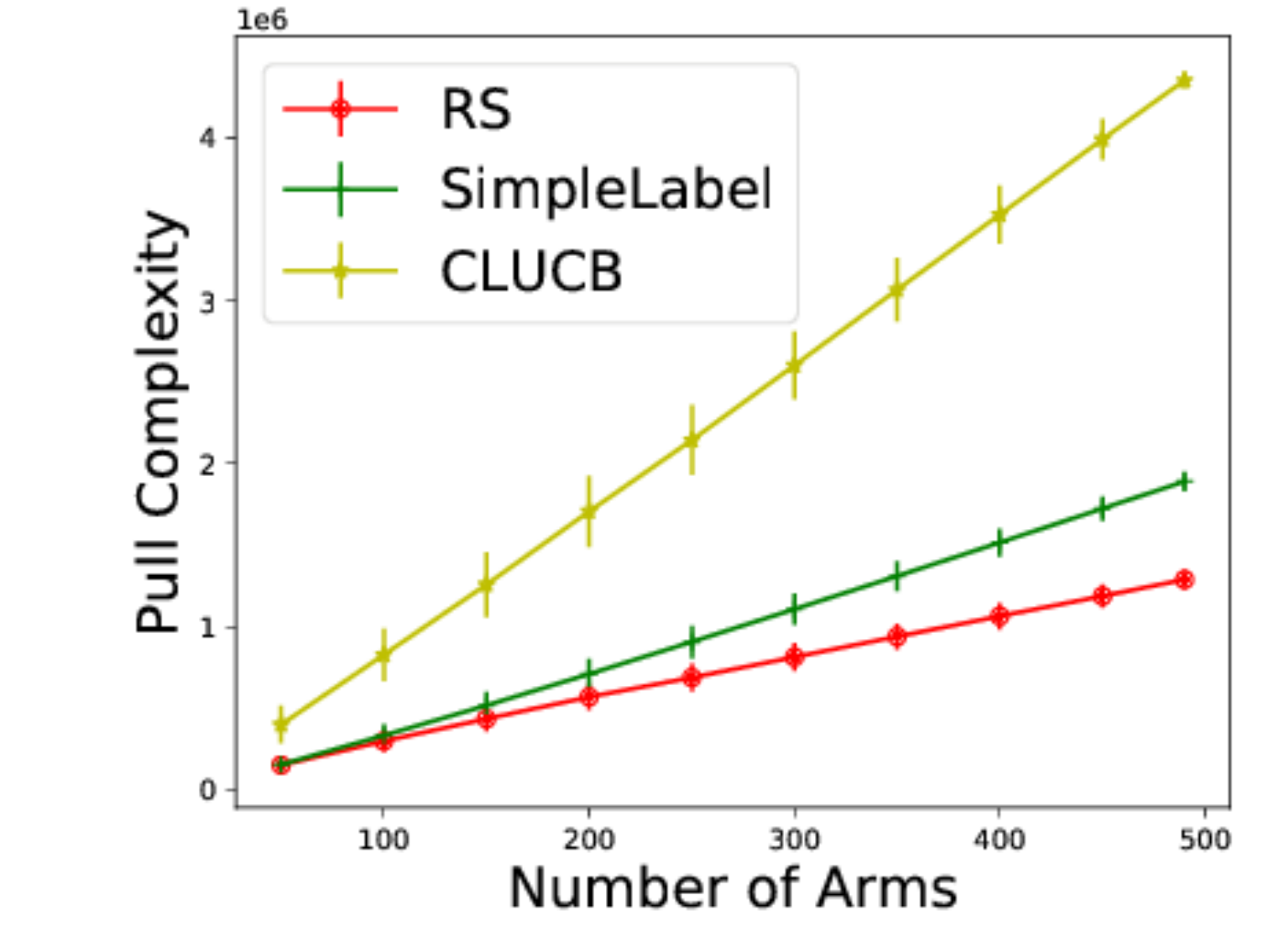}
		\caption{Linear Link Function}
	\end{subfigure}
	\begin{subfigure}[b]{0.35\textwidth}
		\centering
		\includegraphics[width=\textwidth]{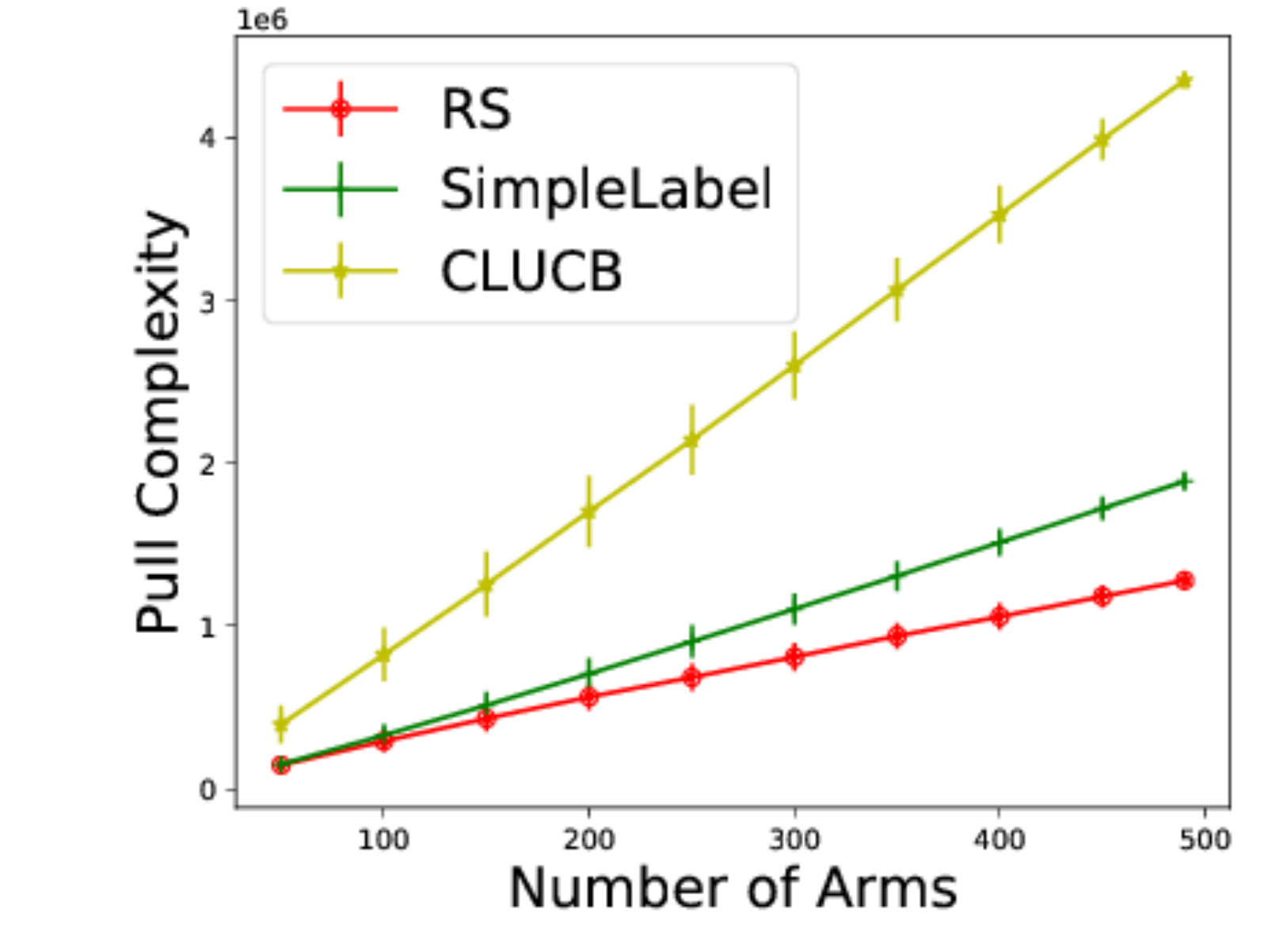}
		\caption{BTL Link Function}
	\end{subfigure}	
	\caption{Empirical results comparing \algoname and other baselines. Error bars represent standard deviation across 500 experiments.\label{fig:expr_result_real}}
	\lsp
\end{figure}

\textbf{Results on real-world data.} Finally, we compare the pull complexity between \algoname and the pull-only baselines on real-world data in Figure \ref{fig:expr_result_real}. \algoname still performs better than both baselines for the real data, but the advantage of \algoname over the baselines are lower than on synthetic data. This is possibly because the data contains many passages near the boundary (i.e., grade 6 and 7), and \algoname have to use pulls to identify their label. We verify this empirically in the appendix.


\msp
\section{Conclusion}
\ssp

We formulate a new setting of the Thresholding Bandit Problem with Dueling Choices, and provide the \algoname algorithm, along with upper and lower bounds on its performance. For future work, it would be interesting to tighten the upper and lower bounds to match them; We believe it should be possible to improve the lower bound by randomizing the arms closest to the threshold. It would also be interesting to develop algorithms adapting to varying noise levels in comparisons.


\ssp
\subsubsection*{Acknowledgements}
\ssp
This work has been supported in part by DARPA FA8750-17-2-0130, NSF
CCF-1763734 and IIS-1845444, and AFRL FA8750-17-2-0212.

\clearpage
\bibliography{yichongref}
\clearpage
\appendix
\onecolumn
\section{Additional Experiment Details}

\textbf{Method to initialize $\comperr_0$.} The method to find the initial $\comperr_0$ is stated in Algorithm \ref{algo:init_gamma}. We lower $\comperr_0$ iteratively until we find $\max \estprob{\bs}-\min \estprob{\bs}\geq 2\comperr_0$. This criteria is set so that we are likely to find separable arms in subsequent binary searches.

\begin{algorithm}[htb!]
	\caption{Initialize $\comperr_0$}
	\label{algo:init_gamma}
	\begin{algorithmic}[1]        
		\State $\comperr_0\leftarrow 0.1$
		\While{True}
		\While{$\exists \bs\in \workset, \ncomp{\bs}\leq \frac{1}{\comperr_0^2} \log\left(\frac{8|S|(\itercount+1)^2}{\tol}\right) $} \label{step:rank_start}
		\For{$\bs\in \workset$}
		\State Draw $\bs'\in [\narm]$ uniformly at random, and compare arm $\bs$ with arm ${\bs'}$
		\State If arm $\bs$ wins, $\nwin{\bs}\leftarrow \nwin{\bs}+1$
		\State $\ncomp{\bs}\leftarrow \ncomp{\bs}+1$
		\EndFor
		\EndWhile
		\State Compute $\estprob{\bs}\leftarrow \nwin{\bs}/\ncomp{\bs}$ for all $\bs\in \workset$ 
		\If{$\max \estprob{\bs}-\min \estprob{\bs}<2\comperr_0$}
		\State $\comperr_0\leftarrow \comperr_0/1.1$
		\Else 
		\State break
		\EndIf
		\EndWhile
		\Ensure{$\comperr_0$}
	\end{algorithmic}
\end{algorithm}


%
%
\begin{figure*}
	\centering
		\begin{subfigure}[b]{0.33\textwidth}
		\centering
		\includegraphics[width=\textwidth]{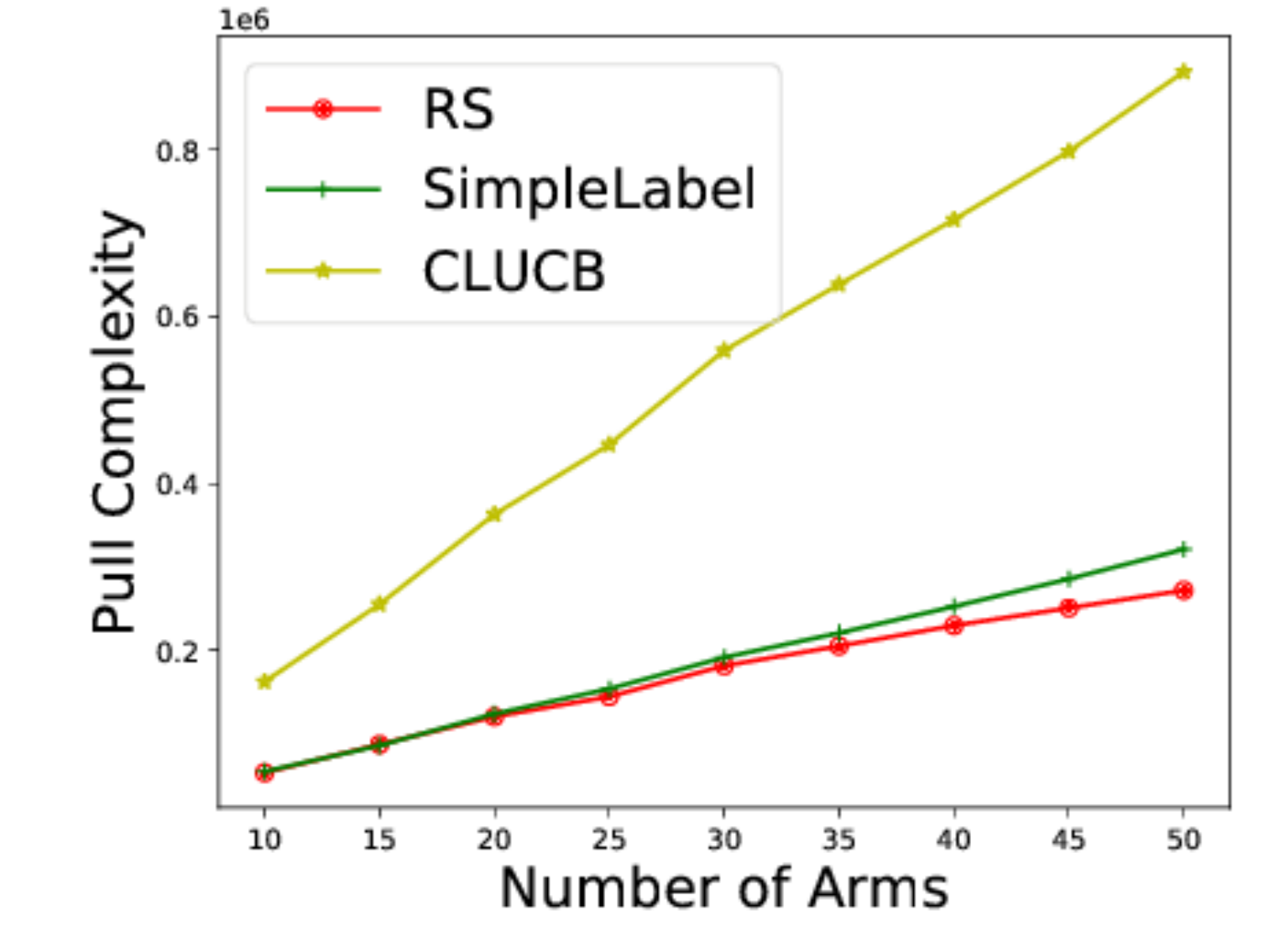}
		\caption{uniform\label{fig:uniform}}
	\end{subfigure}
\hfill
	\begin{subfigure}[b]{0.33\textwidth}
		\centering
		\includegraphics[width=\textwidth]{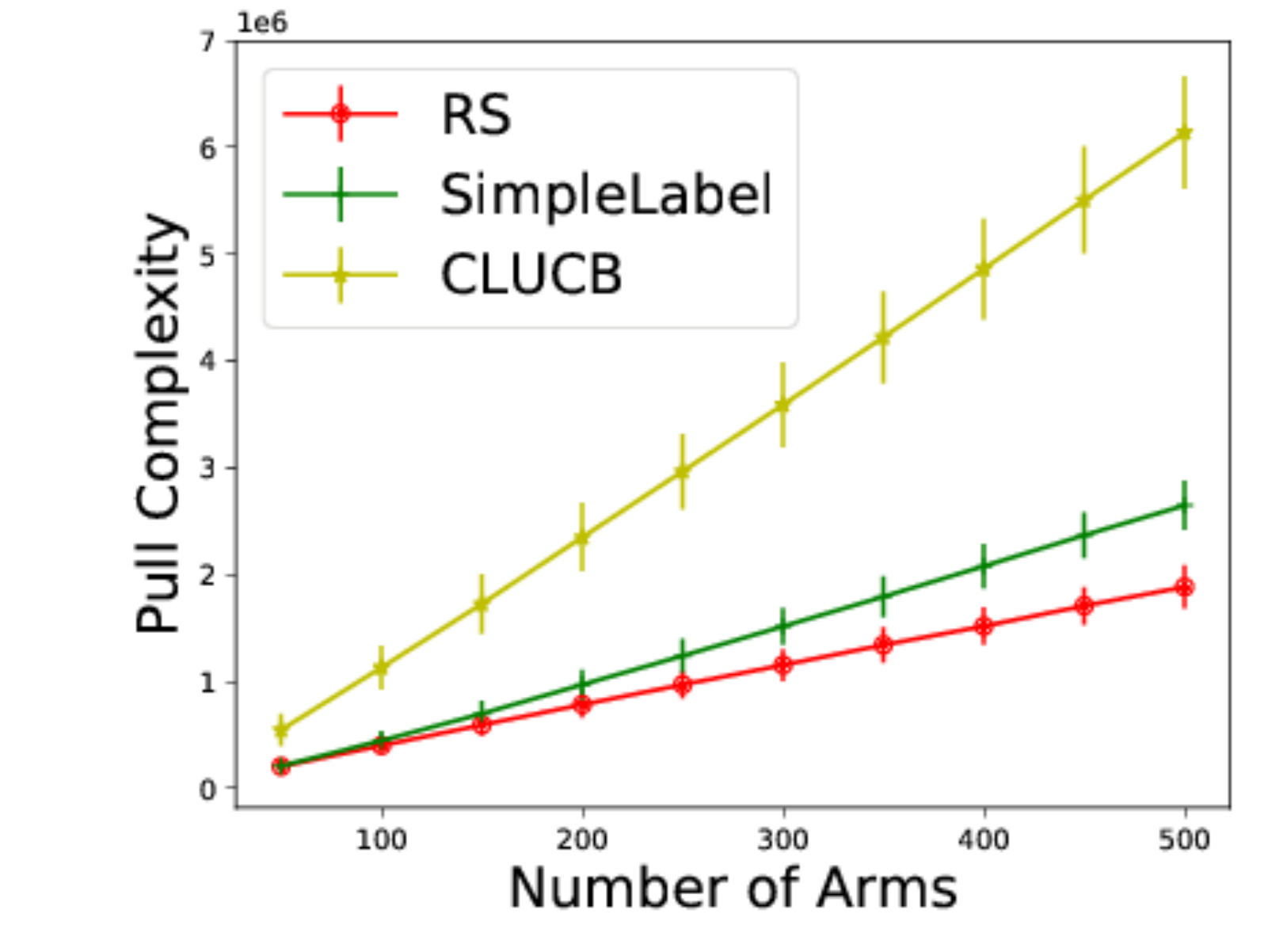}
		\caption{12groups}
	\end{subfigure}%
		\hfill
	\begin{subfigure}[b]{0.33\textwidth}
		\centering
		\includegraphics[width=\textwidth]{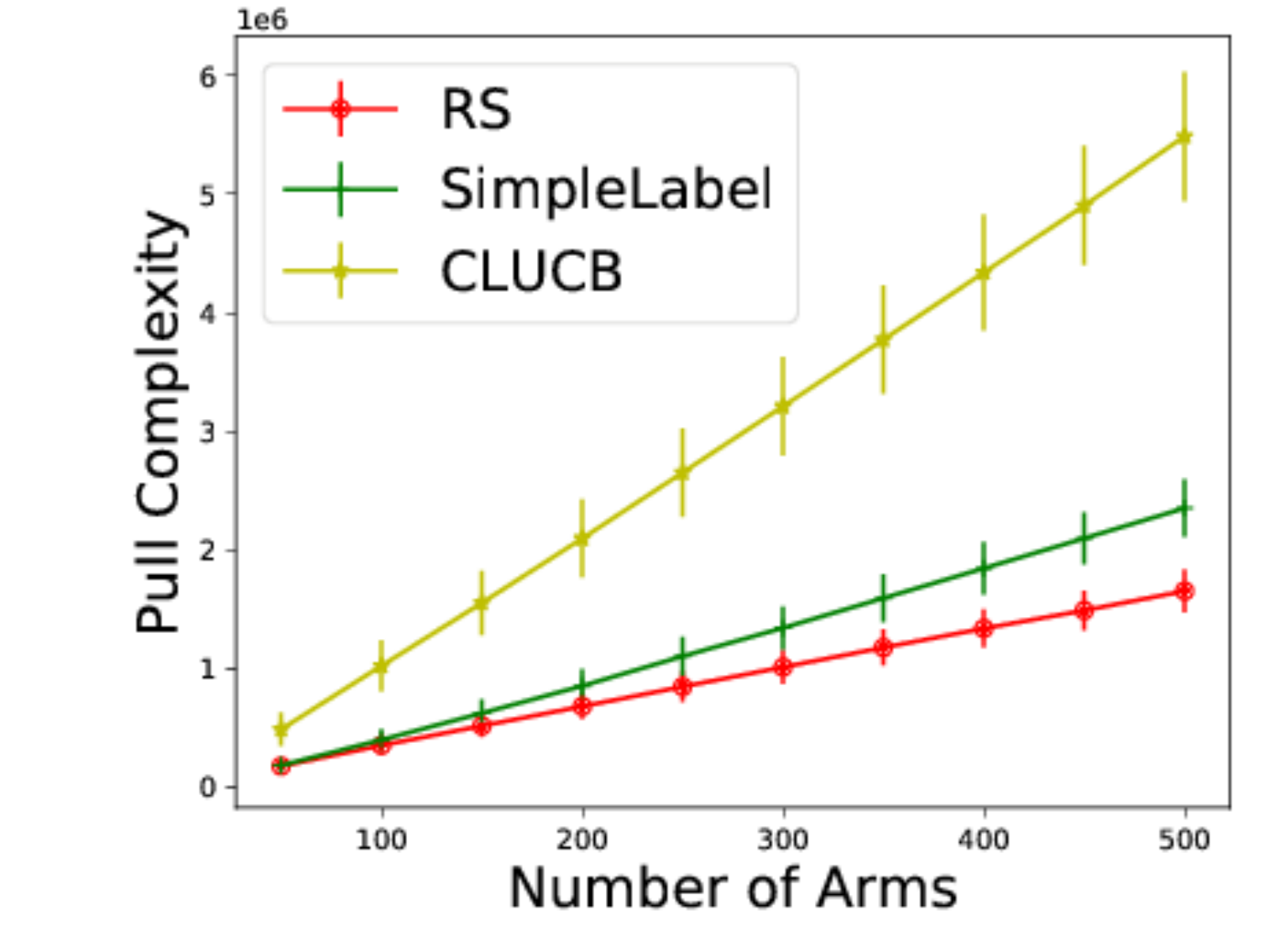}
		\caption{4groups}
	\end{subfigure}
	
	\caption{Empirical results comparing \algoname and other baselines under the 12groups and 4groups setting for pull complexity. Error bars represent standard deviation across 500 experiments.\label{fig:12groups}}
	\lsp
\end{figure*}
\textbf{Additional Synthetic Experiments.} 
In addition to the settings we consider in Section \ref{sec:expr}, we also test the uniform reward distribution:\\
\emph{Uniform}: The means are simply uniformly random in $[0,1]$.
Also to verify our observations on real data (see Section \ref{sec:expr_res}), we compare \algoname with the baselines for pull complexity for the two following setups:\\
\emph{12groups}: The means are uniformly randomly picked from $[1/13, 2/13,...,12/13]$. This simulates the reading difficulty distribution;\\
\emph{4groups}: The means are randomly picked by $\Pr[\mu_i=1/13]=\Pr[\mu_i=12/13]=5/12$ and $\Pr[\mu_i=6/13]=\Pr[\mu_i=7/13]=1/12$. This only keeps the arms close to the boundary and makes the other arms further from the boundary.\\
Results are depicted in Figure \ref{fig:12groups}. For uniform rewards(Figure \ref{fig:uniform}), \algoname achieves a slightly better performance than SimpleLabel, much better than CLUCB. This situation can hardly be improved by using comparisons, since identifying the labels of the hardest arms is almost as difficult as identifying the labels of all the arms. Still, RS can outperform baselines by a small margin since it pulls fewer arms.

For \emph{12groups} and \emph{4groups} (Figure \ref{fig:12groups}b,c) we obtain a similar performance gain as in the real data setting, suggesting that the arms close to the boundary are increasing the cost of \algoname. We note that the pulls of \algoname is necessary since there is no other way to identify the labels of arms with means $6/13$ and $7/13$.

\begin{figure*}[ht!]
	\centering
	\begin{subfigure}[b]{0.33\textwidth}
		\centering
		\includegraphics[width=\textwidth]{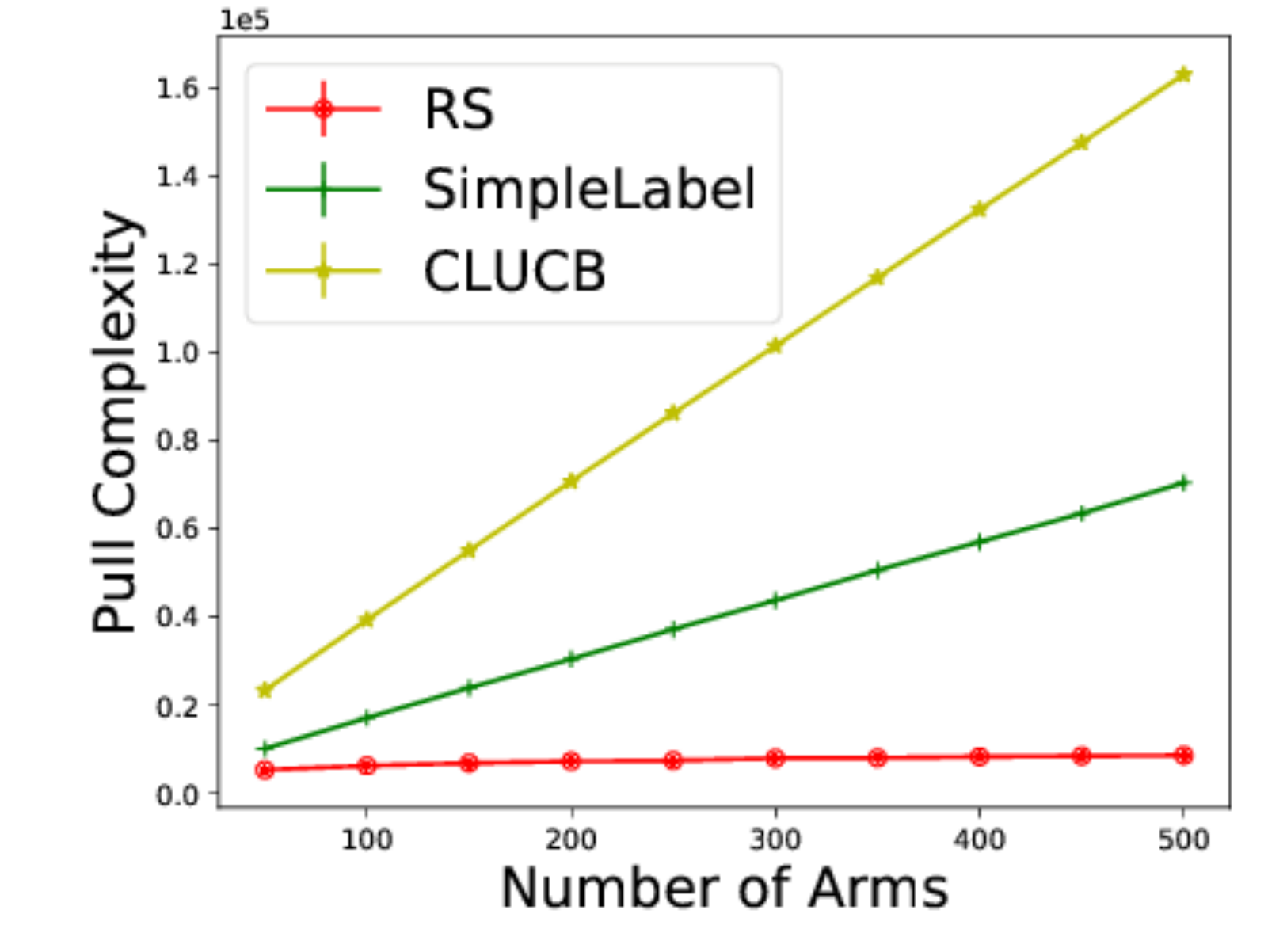}
		\caption{harmonic}
	\end{subfigure}%
	\begin{subfigure}[b]{0.33\textwidth}
		\centering
		\includegraphics[width=\textwidth]{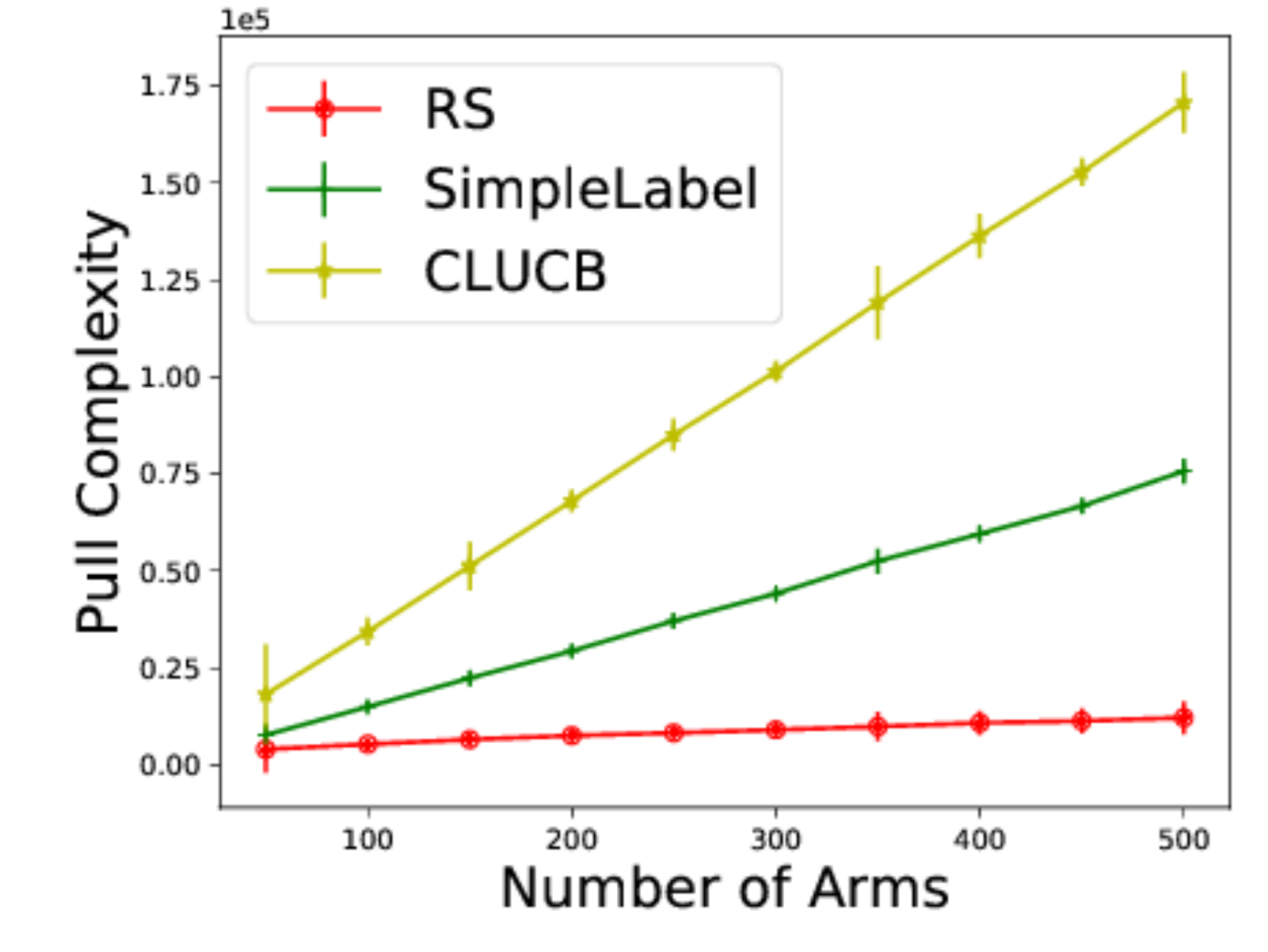}
		\caption{exponential}
	\end{subfigure}
	\begin{subfigure}[b]{0.33\textwidth}
		\centering
		\includegraphics[width=\textwidth]{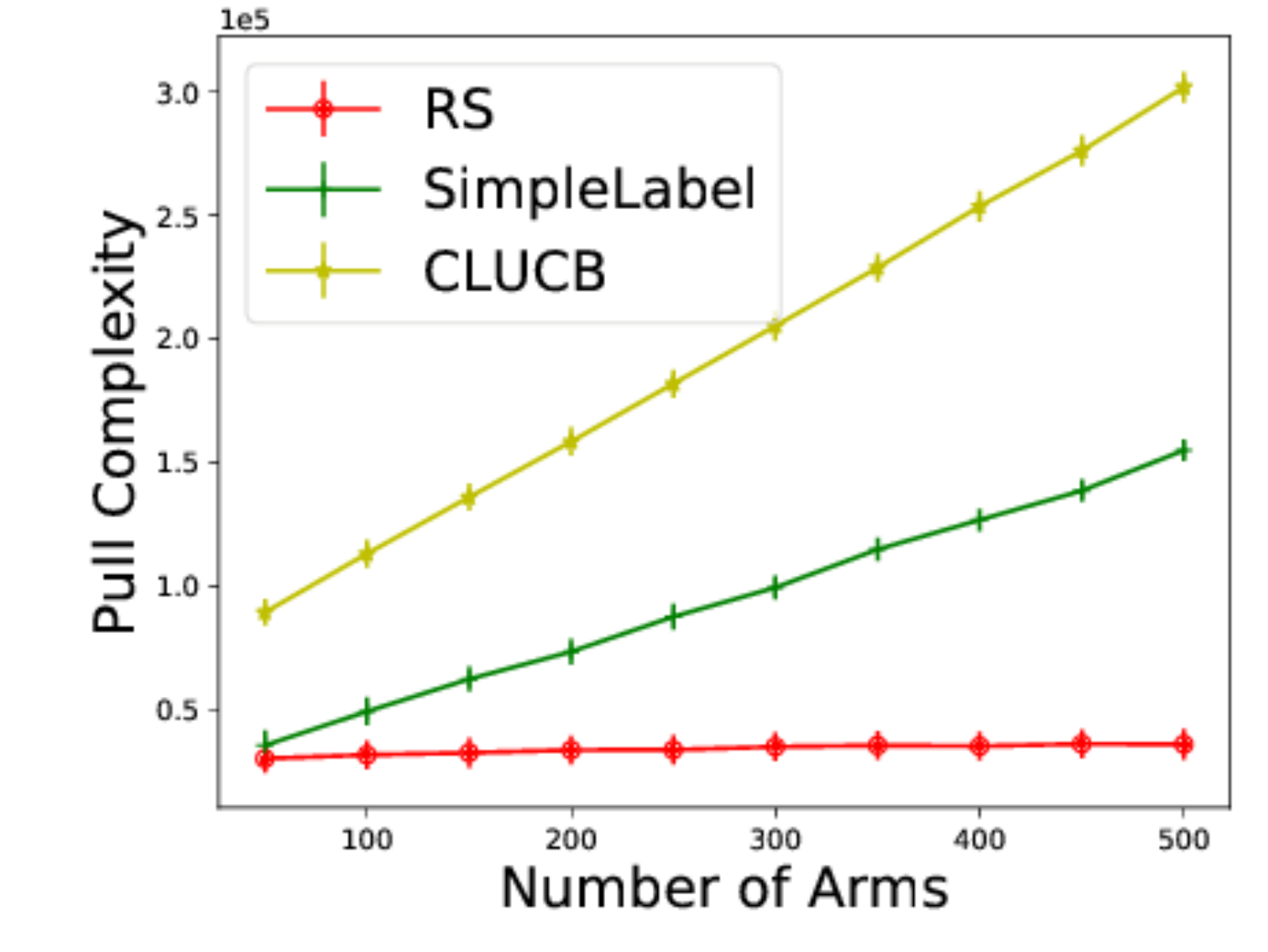}
		\caption{3 groups}
	\end{subfigure}
	
	\caption{Empirical results comparing \algoname and other baselines under BTL model for pull complexity. Error bars represent standard deviation across 500 experiments.\label{fig:expr_result_btl}}
	\lsp
\end{figure*}
\textbf{Results on BTL model.} We compare \algoname with the baselines under the same synthetic data but with the BTL model for comparisons. The results for pull complexity is in Figure \ref{fig:expr_result_btl} and duel complexity in Figure \ref{fig:expr_result_duelonly_btl}. The results are generally very similar to the linear link function case, but with a larger duel complexity. As in the linear link function case, \algoname exhibits a better performance in both pull and duel complexity than all the other baselines.

\begin{figure}[htbp!]
	\centering
	\begin{subfigure}[b]{0.35\textwidth}
		\centering
		\includegraphics[width=\textwidth]{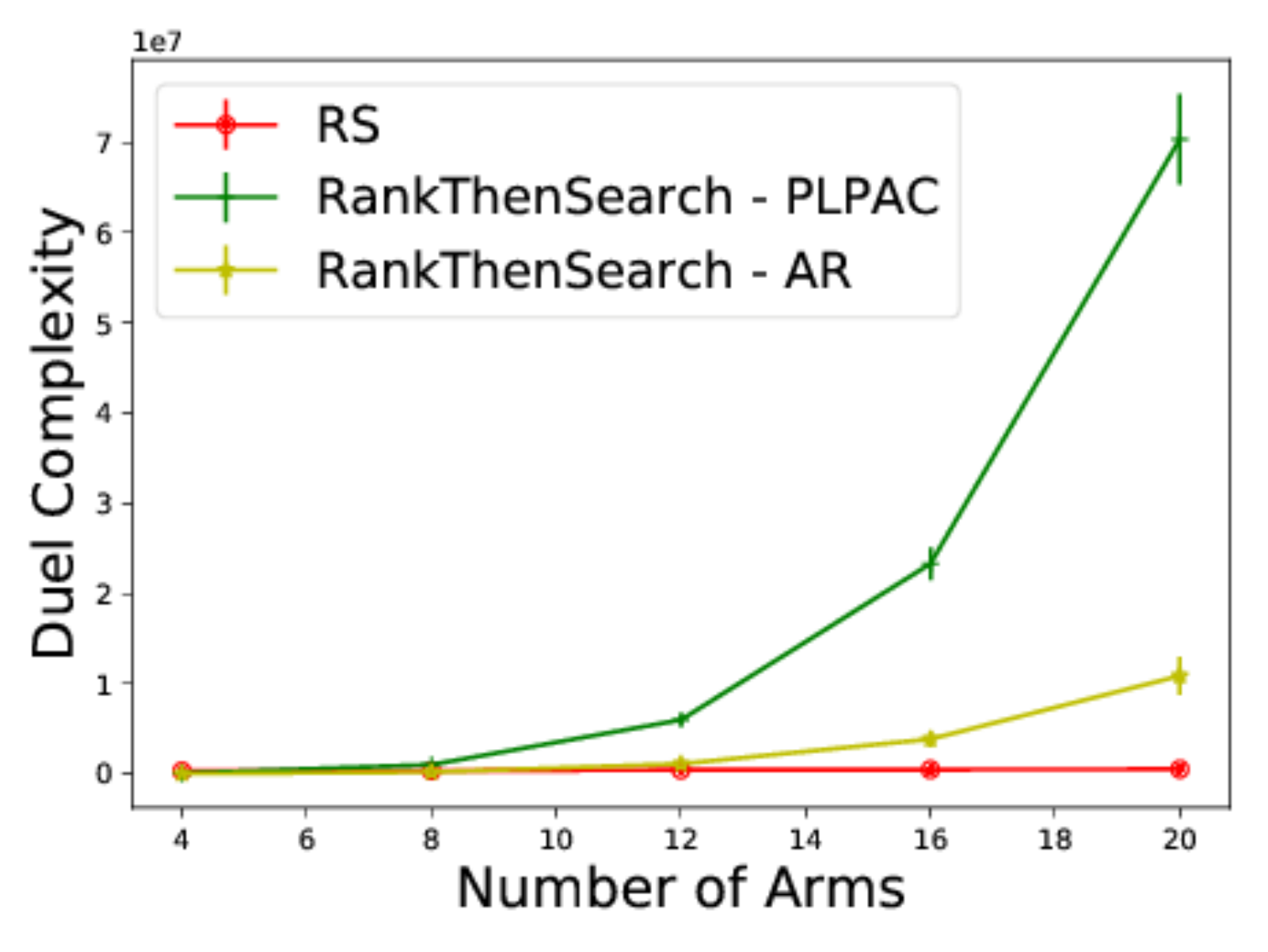}
		\caption{harmonic}
	\end{subfigure}%
	\begin{subfigure}[b]{0.35\textwidth}
		\centering
		\includegraphics[width=\textwidth]{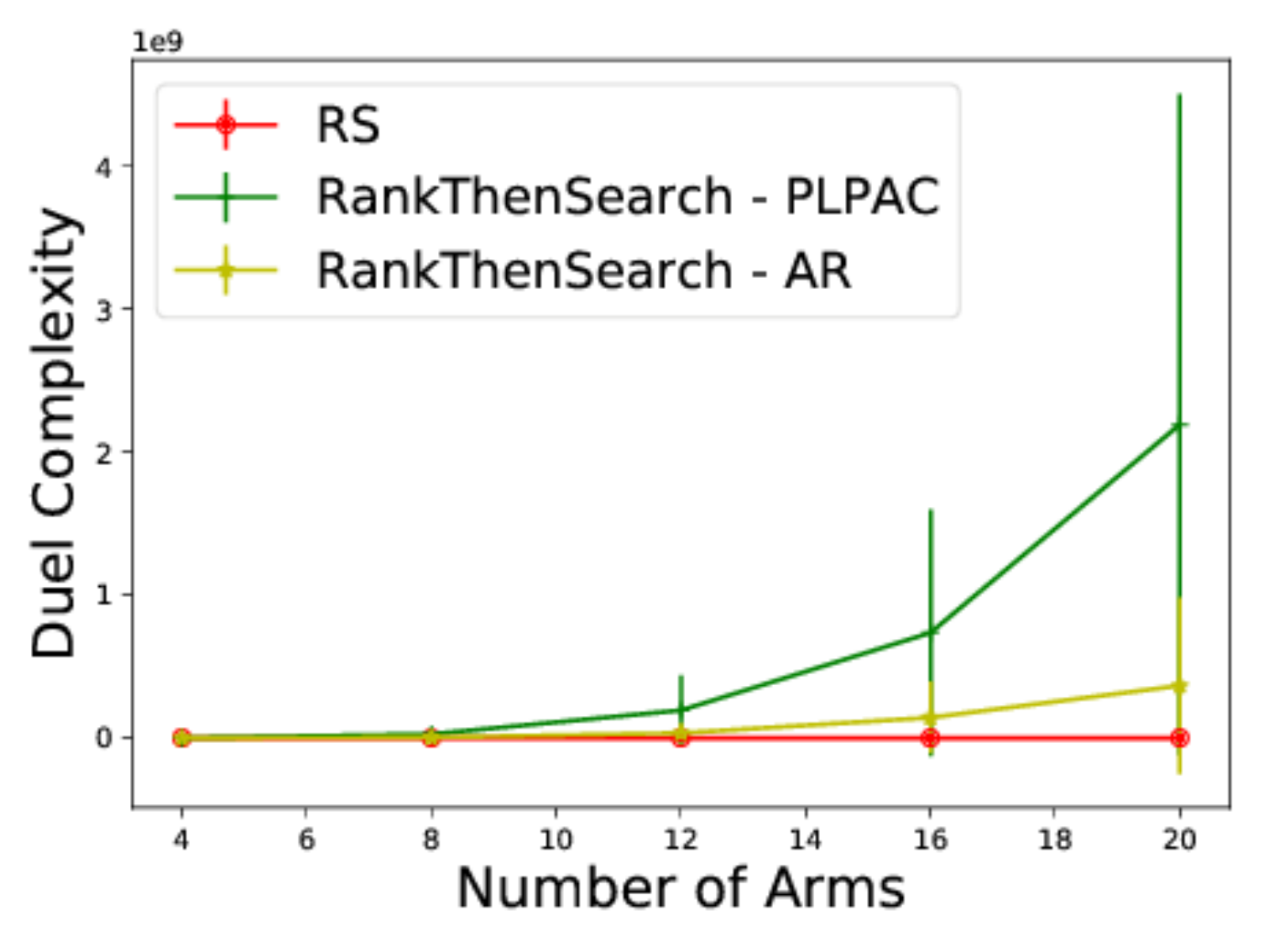}
		\caption{exponential}
	\end{subfigure}
	
	\caption{Empirical results comparing \algoname and RankThenSearch under BTL model for duel complexity. PLPAC is short for PLPAC-AMPR. \label{fig:expr_result_duelonly_btl}}
\end{figure}



\section{Proofs}

\subsection{Proof of Theorem \ref{thm:upper}}
First we show that with high probability our confidence interval in Algorithm \ref{algo:ranksearch} and \ref{algo:figure_out_label} bounds $\pbbeat{i}$ and $\sammean{i}$.
\begin{lemma}\label{lem:concentration}
	With probability $1-\tol$ the following holds:
	\begin{itemize}
		\item At step \ref{step:rank} in Algorithm \ref{algo:ranksearch} we have $|\pbbeat{\samiter}-\estprob{\samiter}|\leq \comperr_\itercount$ for all $\samiter\in S$ and all $t$;
		\item At step \ref{step:delta0} in Algorithm \ref{algo:figure_out_label} we have $|\sammean{\samiter}-\estmean{\samiter}|\leq \gamma$ for all arms $\samiter$ that are passed to Algorithm \ref{algo:figure_out_label}. 
	\end{itemize}
\end{lemma}
\begin{proof}
	The lemma follows from standard concentration inequality. Using Hoeffding's inequality and a union bound we know that in each round of Algorithm \ref{algo:ranksearch} we have
	\[\Pr[\exists \samiter, |\pbbeat{\samiter}-\estprob{\samiter}|> \comperr_\itercount]\leq |S|\exp(-2\ncomp{\samiter}\cdot \comperr_\itercount^2)\leq \frac{\tol}{4t^2}.\]
	Sum it up we have $|\pbbeat{\samiter}-\estprob{\samiter}|\leq \comperr_\itercount$ holds for all $\samiter\in \workset$ and all rounds $t$ with probability at most $\tol/2$.
	
	Similarly, from Hoeffding's inequality for sub-Gaussian random variables and a union bound we have for any run of \figurelabelname,
	\begin{align*}
	\Pr[\exists t, |\sammean{\samiter}-\estmean{\samiter}|> \gamma]&\leq \sum_{t=0}^\infty \exp(-\frac{\gamma^2}{2R^2})\\
	&\leq \sum_{t=0}^\infty \frac{\tol_1}{4(t+1)^2}\leq \tol_1.
	\end{align*} 
	Now sum the probability over all runs of \figurelabelname we have
	\[\Pr[\text{Every \figurelabelname is correct}]=\sum_{t=0}^\infty \frac{\tol}{4(t+1)^2}\log |S|\cdot \frac{1}{\log |S|}\le \tol/2. \]
	The lemma follows from another union bound.
\end{proof}
We now assume the event in Lemma \ref{lem:concentration} happens. Now we can show that we never make a mistake when we estimate labels in Algorithm~\ref{algo:figure_out_label} using direct pulls. Firstly, upon termination of \figurelabelname we have $|\estmean{\samiter}-\thres|> \gamma$. Not losing generality, suppose we have $\hat{\dlabel}_{\samiter}=1$ as the output. Then we have $\estmean{\samiter}-\thres>\gamma$, and thus $\sammean{\samiter}>\thres$, so $\samiter\in \posset$. Similarly we do not make a mistake when $\hat{\dlabel}_{\samiter}=0$. 

To show the correctness when we infer labels in step \ref{step:infer_pos} and \ref{step:infer_neg} in Algorithm \ref{algo:ranksearch}, we first need the following lemma for binary search in an arbitrary noisy sequence:
\begin{lemma}\label{lemma:bs_prop}
	\bsalgoname always returns within $\lceil \log (|\workset|)+1\rceil$ iterations, and the first output $\bsvar$ satisfies i) $\hat{\dlabel}_{\bs_{\bsvar+1}}=1$ if $\bsvar<|S|$; and ii) $\hat{\dlabel}_{\bs_{\bsvar}}=0$ if $\bsvar>0$.
\end{lemma}
\begin{proof}
	Firstly, Algorithm \ref{algo:binary_search} always terminates, because $\bsvar=\lceil(\bsvar_{\min}+\bsvar_{\max})/2\rceil$ satisfies $\bsvar_{\max}-\bsvar_{\min}\geq 2\max\{\bsvar-\bsvar_{\min}, \bsvar_{\max}-\bsvar \}$. For simplicity, define imaginary labels $\hat{\dlabel}_0=0,\hat{\dlabel}_{|S|+1}=1$. We prove by induction that we always have $\hat{\dlabel}_{\bs_{\bsvar_{\min}}}=0$ and $\hat{\dlabel}_{\bs_{\bsvar_{\max}+1}}=1$. This is true for the first iteration; for subsequent iterations, if we move to the left (Line \ref{step:moveleft}) we have $\hat{\dlabel}_{\bs_{\bsvar}}=\hat{\dlabel}_{\bs_{\bsvar_{\max}+1}}=1$; if we move the right (Line \ref{step:moveright}) we have $\hat{\dlabel}_{\bs_{\bsvar}}=\hat{\dlabel}_{\bs_{\bsvar_{\min}}}=0$. Therefore the claim holds. Note that upon termination we must have $\bsvar_{\max}=\bsvar_{\min}$. The lemma then follows from the claim. 
\end{proof}
Now if we let $\hat{\dlabel}_{\bs}=1$ in step \ref{step:infer_pos} in Algorithm \ref{algo:ranksearch}, we have $\estprob{\bs}-\estprob{{\bs_{\bsvar+1}}}\geq 2\comperr_{\itercount}$, and therefore $\pbbeat{\bs}>\pbbeat{{\bs_{\bsvar+1}}}$. Since, we have $\hat{\dlabel}_{\bs_{\bsvar+1}}=1$ from Lemma \ref{lemma:bs_prop} and its label is estimated correctly by Algorithm \ref{algo:figure_out_label}, ${\dlabel}_{\bs_{\bsvar+1}}=1$ and thus ${\bs_{\bsvar+1}}\in \posset$.
Since ${\bs_{\bsvar+1}}\in \posset$, $\pbbeat{\bs_{\bsvar+1}} \geq \pbbeat{j}$ for all $j\in \negset$ and same holds for $\pbbeat{\bs}>\pbbeat{{\bs_{\bsvar+1}}}$ meaning ${\bs}\in \posset$. Similarly we do not make a mistake on step \ref{step:infer_neg}.

Now we consider the number of duels taken to infer when any arm $\samiter=\armset\setminus \{\upthres,\lowthres\}$ is in $\overline{\workset}$ or $\underline{\workset}$ and hence is eliminated from further duels. Not losing generality, suppose $\samiter\in \posset$, and thus $\sammean{\samiter}> \thres$. We show that the arm ${\samiter}$ is eliminated from further duels when we have $4\comperr_\itercount< \diffcomp_{\samiter}$. Suppose we have $i \not\in \overline{\workset}$ i.e. $\estprob{\bs_{\bsvar+1}}\geq \estprob{\samiter}-2\comperr_\itercount$ at the end of the binary search in round $\itercount$. Let $j=\argmax_{j\in \posset} \min\{\pbbeat{j}-\pbbeat{\lowthres}, \pbbeat{\samiter}-\pbbeat{j} \}$ be the maximizer to obtain $\diffcomp_\samiter$.

By Lemma \ref{lem:concentration} and definition of $\diffcomp_\samiter$ we have \[\estprob{j}\leq \pbbeat{j}+\comperr_{\itercount}\leq \pbbeat{\samiter}-\diffcomp_\samiter+\comperr_{\itercount}< \pbbeat{\samiter}-3\comperr_{\itercount}\leq \estprob{\samiter}-2\comperr_{\itercount},\] 
so $\estprob{j}<\estprob{\samiter}-2\comperr_{\itercount}\leq \estprob{\bs_{\bsvar+1}}$. So arm $j$ is ranked before arm ${\bs_{\bsvar+1}}$; and since $\hat{\dlabel}_{\bs_{\bsvar}}=0$ by Lemma~\ref{lemma:bs_prop}, we have ${\bs_{\bsvar}}\not\in \posset$ since its label is estimated correctly by Algorithm \ref{algo:figure_out_label}, and therefore arm $j$ is ranked no later than arm ${\bs_{\bsvar}}$, thus $\estprob{j}\leq \estprob{\bs_{\bsvar}}$. However, from definitions of $\diffcomp_\samiter$ and arm $\lowthres$, we have
\[\pbbeat{j}\geq \pbbeat{\lowthres}+\diffcomp_\samiter\geq \pbbeat{\bs_{\bsvar}}+4\comperr_\itercount. \]
And therefore by Lemma \ref{lem:concentration} we have $\estprob{j}\geq \pbbeat{j}-\comperr_{\itercount}\geq \pbbeat{\bs_{\bsvar}}+3\comperr_{\itercount} \geq  \estprob{\bs_{\bsvar}}+2\comperr_\itercount$, which makes a contradiction. Therefore we will have $\estprob{\bs_{\bsvar+1}}< \estprob{\samiter}-2\comperr_\itercount$ i.e. arm $\samiter \in \overline{S}$, and arm ${\samiter}$ will be excluded from $\workset$ in iteration $\itercount$. In a similar way we can argue that for $\samiter\in \negset$, it is excluded from $\workset$ when $\diffcomp_{\samiter}>4\comperr_t$.

Therefore we would need 
$\frac{\log(8|S|t^2/\tol)}{(\diffcomp_{\samiter}/4)^2}$ duels to eliminate arm ${\bs}$ from further duels. Sum this over all arms $\samiter$ and use the fact that $t\leq \log(1\comperr^*)$, we get the number of duels is $O(\compcpl\log(\frac{K\log(1/\comperr^*)}{\tol}))$ to identify all arms except $\{\lowthres,\upthres\}$. When every arm $i\in \armset\setminus \{\upthres,\lowthres\}$ has been given a label, $\upthres,\lowthres$ will be given a label during binary search.

Now we bound the number of direct pulls. We figure out the label of arm ${\samiter}$ when $2\gamma\geq |\sammean{\samiter}-\thres|$ in Algorithm \ref{algo:figure_out_label}. Therefore for each sample we need $2\sqrt{R^2\frac{2\log(2/\tol_0)}{T}}\leq |\sammean{\samiter}-\thres|$ pulls; note that we only require pulls during binary search. Each binary search runs Algorithm \ref{algo:figure_out_label} for at most $\log \narm$ times, and we do $\log(1/\comperr^*)$ times of binary search. Combining these terms we get the number of pulls.


\subsection{Proof of Theorem \ref{thm:lower}}

Our proof borrows ideas from \cite{heckel2016active} but adapting to the dueling-choice case.
Not losing generality, suppose $\samiterint\in \posset$; the proof for $\samiterint\in\negset$ is similar.
We first use a lemma from bandit literature \citep{kaufmann2016complexity} that links KL divergence with error probability. Let $\nu=\{\nu_j\}_{j=1}^m$ be a collection of $m$ probability distributions supported on $\mathbb{R}$. Consider an algorithm $\algo$ that selects an index $i_t\in [m]$ and receives an independent draw $X$ from $\nu_i$. $i_t$ only depends on its past observations; i.e., $i_t$ is $\mathcal{F}_{t-1}$ measurable, where $\mathcal{F}_t$ is the $\sigma$-algebra generated by $i_1,X_{1},...,i_t,X_{t}$. Let $\chi$ be a stopping rule of $\algo$ that determines the termination of $\algo$. We assume that $\chi$ is measurable w.r.t $\mathcal{F}_t$ and $\Pr[\chi<\infty]=1$. Let $Q_i(\chi)$ be the number of times that $\nu_i$ is selected by $\algo$ until termination. For any $p,q\in (0,1)$, let $d(p,q)=p\log \frac{p}{q}+(1-p)\log\frac{1-p}{1-q}$ be the KL divergence between two Bernoulli distributions with parameter $p,q$. We use the following lemma:
\begin{lemma}[\citep{kaufmann2016complexity}, Lemma 1]\label{lem:mab_lower}
	Let $\nu=\{\nu_j\}_{j=1}^m, \nu'=\{\nu'_j\}_{j=1}^m$ be two collections of $m$ probability distributions on $\mathbb{R}$. For any event $\event\in \mathcal{F}_{\chi}$ with $\Pr_{\nu}[\event]\in (0,1)$ we have
	\begin{equation}
	\sum_{i=1}^m \E_{\nu}[Q_i(\chi)]KL(\nu_i,\nu_i')\geq d(\Pr_{\nu}[\event],\Pr_{\nu'}[\event]).
	\end{equation}
\end{lemma}
Now, define the event $\event$ to be the event that $\algo$ succeeds under $\prefmat$ and $\meanvec$, i.e., $\event\equiv \left\{\posset=\estposset, \negset=\estnegset \right\}$. Under the relation $\prefmat_{\samiter\samiterp}=1-\prefmat_{\samiterp\samiter}$ the comparison is uniquely defined by the probabilities $\{\prefmat_{\samiter\samiterp}, 1\le \samiter<\samiterp\leq \narm \}$; and pull is uniquely defined by the mean vector $\meanvec$. For any two arms $\samiter, \samiterp$, let $\dueltime{\samiter\samiterp}(\chi)$ be the number of times that arms $\samiter$ and $\samiterp$ duel before stopping. Therefore for two problem settings $(\prefmat,\meanvec)$ and $(\prefmat',\meanvec')$, by Lemma \ref{lem:mab_lower} we have
\begin{equation}
\sum_{\samiter=1}^\narm \sum_{\samiterp=\samiter+1}^\narm \E_{\prefmat,\meanvec}[\dueltime{\samiter\samiterp}]d(\prefmat_{\samiter\samiterp},\prefmat'_{\samiter\samiterp})+\frac{1}{2R^2}\sum_{\samiter=1}^{\narm}(\sammean{\samiter}-\sammean{\samiter}')^2\geq d(\Pr_{\prefmat,\meanvec}[\event],\Pr_{\prefmat',\meanvec'}[\event]).\label{eqn:dpr}
\end{equation}
The second term in (\ref{eqn:dpr}) follows from the KL divergence between Gaussian variables.
We now construct another feasible profile $(\prefmat',\meanvec')$ and that $\sammean{\samiterint}<\thres$ and that $\pbbeat{\samiterint}'<\pbbeat{\samiterp}'$ for any $\samiterp\in \posset$ according to $\prefmat'$. Therefore in this case $\samiterint\not\in \posset(\prefmat',\meanvec')$, where $\posset(\prefmat',\meanvec')$ is the set of arms with reward larger than $\thres$ under $\prefmat',\meanvec'$. Since $\algo$ succeeds with probability $1-\tol$ for any problem setting, we have $\Pr_{\prefmat,\meanvec}[\event]\geq 1-\tol$ and $\Pr_{\prefmat',\meanvec'}[\event]\leq \tol$. Therefore
\[d(\Pr_{\prefmat,\meanvec}[\event],\Pr_{\prefmat',\meanvec'}[\event])\geq d(\tol,1-\tol)\geq \log\frac{1}{2\tol}, \]
which holds for $\tol\leq 0.15$. 

We now specify $\prefmat',\meanvec'$. Let
\[\prefmat'_{\samiter\samiterp}=\begin{cases}
\prefmat_{\samiterint\samiterp}-(\pbbeat{\samiterint}-\pbbeat{\upthres}), & \text{if } \samiter=\samiterint, \samiterp\ne \samiterint,\\
\prefmat_{\samiterint\samiterp}+(\pbbeat{\samiterint}-\pbbeat{\upthres}), & \text{if } \samiterp=\samiterint, \samiter\ne \samiterint,\\
\prefmat_{\samiterint\samiterp} & \text{otherwise.} \\
\end{cases} \]
and $\sammean{\samiterint}'=2\thres-\sammean{\samiterint}$. It is easy to see that $\sammean{\samiterint}'\leq \thres$, and therefore $\samiterint\not\in \posset(\prefmat',\meanvec')$. We now show that the profile $\prefmat',\meanvec'$ by showing that $\pbbeat{\samiterint}'<\pbbeat{\samiterp}'$. In the new profile we have
\[\pbbeat{\samiterint}'=\frac{1}{\narm-1}\sum_{\samiterp\ne \samiterint} \prefmat'_{\samiterint\samiterp}=\frac{1}{\narm-1}\sum_{\samiterp\ne \samiterint} (\prefmat_{\samiterint\samiterp}-(\pbbeat{\samiterint}-\pbbeat{\upthres}))=\pbbeat{\samiterint}-\pbbeat{\samiterint}+\pbbeat{\upthres}=\pbbeat{\upthres}. \]
For other arms $\samiter \ne \samiterint$ we have
\[\pbbeat{\samiter}'= \frac{1}{\narm-1}\sum_{\samiterp\ne \samiter} \prefmat'_{\samiter\samiterp} =\pbbeat{\samiter}+\frac{1}{\narm-1} (\pbbeat{\samiterint}-\pbbeat{\upthres}). \]
And therefore $\pbbeat{\samiterint}'=\pbbeat{\upthres}<\pbbeat{\samiter}'$ for any $\samiter\in \posset(\prefmat',\meanvec')$, and therefore $(\prefmat',\meanvec')$ is feasible. Also since $\prefmat_{\samiter\samiterp}\in [\frac{3}{8},\frac{5}{8}]$ we have
\[\prefmat'_{\samiter\samiterp}\leq \frac{5}{8}+(\frac{5}{8}-\frac{3}{8})=\frac{7}{8}, \]
and similarly $\prefmat'_{\samiter\samiterp}\geq \frac{1}{8}$. So for any $\samiterp\ne \samiterint$ we have
\begin{equation}
d(\prefmat_{\samiterint\samiterp},\prefmat'_{\samiterint\samiterp})\leq \frac{(\prefmat_{\samiterint\samiterp}-\prefmat'_{\samiterint\samiterp})^2}{\prefmat'_{\samiterint\samiterp}(1-\prefmat'_{\samiterint\samiterp})}=10(\pbbeat{\samiterint}-\pbbeat{\upthres})^2 \label{eqn:boundd}
\end{equation}
Now consider the sums on the LHS of (\ref{eqn:dpr}). Note that $\prefmat'_{\samiter\samiterp}=\prefmat_{\samiter\samiterp}$ when $\samiter\ne \samiterint,\samiterp\ne \samiterint$; and also $\sammean{\samiterint}-\sammean{\samiterint}'=2(\sammean{\samiterint}-\thres)$ and $\sammean{\samiter}-\sammean{\samiter}'=0$ for $\samiter\ne\samiterint$. Combining (\ref{eqn:dpr}) and the uniform bound in (\ref{eqn:boundd}) we have
\begin{align*}
& \sum_{\samiter=1}^\narm \sum_{\samiterp=\samiter+1}^\narm \E_{\prefmat,\meanvec}[\dueltime{\samiter\samiterp}]d(\prefmat_{\samiter\samiterp},\prefmat'_{\samiter\samiterp})+\frac{1}{2R^2}\sum_{\samiter=1}^{\narm}(\sammean{\samiter}-\sammean{\samiter}')^2\\
\leq &\;  10(\pbbeat{\samiterint}-\pbbeat{\upthres})^2\sum_{\samiterp\ne \samiterint} \E_{\prefmat,\meanvec}[\dueltime{\samiterint\samiterp}]+\frac{2(\sammean{\samiterint}-\thres)^2}{\sgpara^2} \E[\pulltime{\samiterint}]\\
=&\;10(\pbbeat{\samiterint}-\pbbeat{\upthres})^2 \E_{\prefmat,\meanvec}[\dueltime{\samiterint}]+\frac{2(\sammean{\samiterint}-\thres)^2}{\sgpara^2} \E[\pulltime{\samiterint}]
\end{align*}
Combining the expectations we get the desired results.
\subsection{Proof of Corollary \ref{col:lower_lowpull}}
The corollary follows directly from Theorem \ref{thm:lower}: For $\samiterint\not\in \{\upthres,\lowthres\}$ we have $\pulltime{\samiterint}=0$, and therefore 
\[\E_{\prefmat,\meanvec}[\dueltime{\samiterint}^{\algo}]\geq \frac{c\log(\frac{1}{2\tol})}{(\compgap{\samiterint})^2}.\]
Sum this over all arm $\samiterint\not\in \{\upthres,\lowthres\}$ we get the desired result.


\subsection{Proof of Proposition \ref{prop:massart}}
Under the link function assumption we have
\begin{align*}
\pbbeat{\upthres}-\pbbeat{\lowthres}&=\frac{1}{\narm-1}\sum_{\samiter\ne \upthres} \linkfunc(\sammean{\upthres}-\sammean{\samiter})-\frac{1}{\narm-1}\sum_{\samiter\ne \lowthres} \linkfunc(\sammean{\lowthres}-\sammean{\samiter})\\
&\geq \frac{1}{\narm-1}\sum_{\samiter=1}^\narm\left[ \linkfunc(\sammean{\upthres}-\sammean{\samiter})-\linkfunc(\sammean{\lowthres}-\sammean{\samiter})\right]\\
&\geq \frac{\narm}{\narm-1} L(\sammean{\upthres}-\sammean{\lowthres})\geq 2L\massartnoise.
\end{align*}
For any $\samiter\in \posset$, use $\samiterp=\upthres$ and we have $\diffcomp_{\samiter}\geq \min\{2L\massartnoise,\compgap{\samiter} \}$, and this holds similarly for $\samiter\in \negset$. Finally for iii), notice that $2L\massartnoise\leq 1$ because otherwise $\sigma(2\massartnoise)>1$, and $\compgap{\samiter}\leq 1$. Thus we have $\diffcomp_{\samiter}\geq 2L\massartnoise\compgap{\samiter}$, and it leads to iii).

\subsection{Proof for Example 1}
The results follow easily from Theorem \ref{thm:upper}: We have $\labelgap{\samiter}=|\sammean{\samiter}-\tau|\geq \frac{1}{6}$ for every arm $\samiter$. Under the linear link function, we have $\pbbeat{\samiter}-\pbbeat{\samiterp}=\Theta(\sammean{\samiter}-\sammean{\samiterp})$, and thus $\diffcomp_{\samiter} = \Omega(1)$ 
for every arm $\samiter\not\in \{l,l+1\}$. Therefore $\dlcpl=O(\narm)$ and $\compcpl=O(\narm)$, and the results follow.
\subsection{Proof for Example 2}
For each $x_\samiter$, we have $\Pr[x\leq 1/4]\leq \tol/(4l)$. Using a union bound, we have that with probability $1-\tol/2$ we have $x_\samiter\leq 1/4 \forall \samiter\in [\narm]$. Let this event be $E_B$. So under $E_B$ all sample means are in $[0,1/4]$ and $[3/4,1]$, so pull-only algorithm requires $\Omega(\narm)$ pulls.

On the other hand, let $x_{(l)}$ and $x_{(l-1)}$ be the $l$-th and $(l-1)$-th order statistic of $x_1,...,x_l$, i.e., the largest and second largest element of $x_1,...,x_l$. Then $x_{(l)}-x_{(l-1)}$ is distributed according to a exponential distribution with parameter $\lambda$. Routine calculation shows that
\[\Pr[x_{(l)}-x_{(l-1)}\geq \frac{-\log(1-\tol/4)}{\lambda}]\geq 1-\tol/4. \]
Plug in $\lambda$ and $E_B$ we have
\[\Pr[x_{(l)}-x_{(l-1)}\geq \frac{-\log(1-\tol/4)}{4\log(4l/\tol)},E_B]\geq 1-\tol/2. \]
Under this event and symmetrically for $l+1\leq \samiter\leq 2l$, we have $\compcpl=O(\narm \log^2 \narm)$; thus $n_{\text{duel}}=O(\narm \log^3\narm)$ and $n_{\text{pull}}=O(\log^2 \narm) $.
\end{document}